\documentclass[11pt]{article}
\usepackage[tmargin=1in, lmargin=1.25in, rmargin=1.25in, bmargin=1.0in]{geometry}
\usepackage{setspace}
\usepackage[utf8]{inputenc} 
\usepackage[backref,colorlinks,citecolor=blue,bookmarks=true]{hyperref}       
\usepackage{mathtools, amssymb, amsthm}
\usepackage[capitalize]{cleveref}
\usepackage[ruled]{algorithm}
\usepackage{algpseudocode}
\usepackage{enumerate}
\usepackage[affil-it]{authblk}

\theoremstyle{plain}
\newtheorem{theorem}{Theorem}[section]
\newtheorem{lemma}[theorem]{Lemma}
\newtheorem{corollary}[theorem]{Corollary}

\newtheorem{assumption}[theorem]{Assumption}
\newtheorem*{claim}{Claim}

\theoremstyle{definition}
\newtheorem{definition}[theorem]{Definition}
\newtheorem{example}[theorem]{Example}

\theoremstyle{remark}

\newcommand*{\N}{{\mathbb{N}}}

\newcommand*{\R}{{\mathbb{R}}}

\newcommand*{\cC}{{\mathcal{C}}}
\newcommand*{\cH}{{\mathcal{H}}}
\newcommand*{\cG}{{\mathcal{G}}}
\newcommand*{\cD}{{\mathcal{D}}}
\newcommand*{\cF}{{\mathcal{F}}}
\newcommand*{\cN}{{\mathcal{N}}}

\newcommand*{\cZ}{{\mathcal{Z}}}
\newcommand*{\cB}{{\mathcal{B}}}
\newcommand{\sur}{\mathsf{sur}}
\newcommand{\sq}{\mathsf{sq}}
\newcommand{\hs}{\mathsf{hs}}
\newcommand{\mon}{\mathsf{mon}}

\let\eps\epsilon
\let\Pr\relax
\DeclareMathOperator*{\Pr}{\mathbb{P}}
\let\Ex\relax
\DeclareMathOperator*{\Ex}{\mathbb{E}}

\DeclarePairedDelimiter{\inn}{\langle}{\rangle}
\DeclarePairedDelimiter{\norm}{\|}{\|}
\DeclarePairedDelimiter{\floor}{\lfloor}{\rfloor}
\newcommand{\dotp}[2]{#1 \cdot #2}
\DeclareMathOperator{\sda}{SDA}
\DeclareMathOperator{\sgn}{sign}
\DeclareMathOperator{\lsgn}{\widetilde{sign}}
\DeclareMathOperator{\relu}{ReLU}
\DeclareMathOperator{\conv}{conv}
\DeclareMathOperator{\diam}{diam}

\renewcommand{\th}{^\text{th}}
\let\hat\widehat
\newcommand{\cfunc}{p}
\newcommand{\lsur}{\ell_\sur}
\newcommand{\Lsur}{L_\sur}

\newcommand{\opt}{\mathsf{opt}}
\newcommand{\cmf}{\mathsf{cmf}}
\newcommand{\cor}{\mathsf{cor}}

\DeclareMathOperator{\poly}{poly}

\title{Statistical-Query Lower Bounds via Functional Gradients}
\author[1]{Surbhi Goel\thanks{Supported by the JP Morgan AI Fellowship.}}
\author[2]{Aravind Gollakota\thanks{Supported by NSF awards AF-1909204, AF-1717896, and a UT Austin Provost's Fellowship.}}
\author[2]{Adam Klivans\thanks{Supported by NSF awards AF-1909204, AF-1717896, and the NSF AI Institute for Foundations of Machine Learning (IFML). Work done while visiting the Institute for Advanced Study, Princeton, NJ.}}
\affil[1]{Microsoft Research NYC}
\affil[2]{Department of Computer Science, University of Texas at Austin}

\date{October 21, 2020}

\begin{document}

\maketitle

\begin{abstract}
  We give the first statistical-query lower bounds for agnostically learning any non-polynomial activation with respect to Gaussian marginals (e.g., ReLU, sigmoid, sign).  For the specific problem of ReLU regression (equivalently, agnostically learning a ReLU), we show that any statistical-query algorithm with tolerance $n^{-(1/\epsilon)^b}$ must use at least $2^{n^c} \epsilon$ queries for some constant $b, c > 0$, where $n$ is the dimension and $\epsilon$ is the accuracy parameter.  Our results rule out {\em general} (as opposed to correlational) SQ learning algorithms, which is unusual for real-valued learning problems. Our techniques involve a gradient boosting procedure  for ``amplifying'' recent lower bounds due to Diakonikolas et al.\ (COLT 2020) and Goel et al.\ (ICML 2020) on the SQ dimension of functions computed by two-layer neural networks.  The crucial new ingredient is the use of a nonstandard convex functional during the boosting procedure.  This also yields a best-possible reduction between two commonly studied models of learning: agnostic learning and probabilistic concepts.
\end{abstract}

\newpage
\section{Introduction}
In this paper we continue a recent line of research exploring the computational complexity of fundamental primitives from the theory of deep learning \cite{goel2019time,yehudai2019power,diakonikolas2020algorithms,yehudai2020learning,diakonikolas2020approximation,frei2020agnostic}. In particular, we consider the problem of fitting a single nonlinear activation to a joint distribution on $\R^n \times \R$.  When the nonlinear activation is ReLU, this problem is referred to as ReLU regression or agnostically learning a ReLU.  When the nonlinear activation is $\sgn$ and the labels are Boolean, this problem is equivalent to the well-studied challenge of agnostically learning a halfspace \cite{kalai2008agnostically}.

We consider arguably the simplest possible setting---when the marginal distribution is Gaussian---and give the first statistical-query lower bounds for learning broad classes of nonlinear activations.  The statistical-query model is a well-studied framework for analyzing the sample complexity of learning problems and captures most known learning algorithms.  For common activations such as ReLU, sigmoid, and $\sgn$, we give complementary upper bounds, showing that our results cannot be significantly improved.  

Let $\cH$ be a function class on $\R^n$, and let $\cD$ be a labeled distribution on $\R^n \times \R$ such that the marginal on $\R^n$ is $D = \cN(0, I_n)$. We say that a learner learns $\cH$ under $\cD$ with error $\epsilon$ if it outputs a function $f$ such that \[ \Ex_{(x, y) \sim \cD}[f(x) y] \geq \max_{h \in \cH} \Ex_{(x, y) \sim \cD}[h(x) y] - \epsilon. \]

One can show that this loss captures 0-1 error in the Boolean case, as well as squared loss in the ReLU case whenever the learner is required to output a nontrivial hypothesis (i.e., a hypothesis with norm bounded below by some constant $c > 0$). (See \cref{app:corr-loss,,app:boolean-loss} for details.)

For ReLU regression, we obtain the following exponential lower bound:
\begin{theorem}
Let $\cH_{\relu}$ be the class of ReLUs on $\R^n$ with unit weight vectors. Suppose that there is an SQ learner capable of learning $\cH_{\relu}$ under $\cD$ with error $\epsilon$ using $q(n, \epsilon, \tau)$ queries of tolerance $\tau$. Then for any $\epsilon$, there exists $\tau = n^{-(1/\epsilon)^{b}}$ such that $q(n, \epsilon, \tau) \geq 2^{n^c} \epsilon$ for some $0 < b, c < 1/2$. That is, a learner must either use tolerance smaller than $n^{-(1/\epsilon)^b}$ or more than $2^{n^c} \epsilon$ queries.
\end{theorem}

Prior work due to Goel et al.~\cite{goel2019time} gave a quasipolynomial SQ lower bound (with respect to correlational queries) for ReLU regression when the learner is required to output a ReLU as its hypothesis.

For the sigmoid activation we obtain the following lower bound: 
\begin{theorem}
	Consider the above setup with $\cH_{\sigma}$, the class of unit-weight sigmoid units on $\R^n$. For any $\epsilon$, there exists $\tau = n^{-\Theta(\log^2 1/\epsilon)}$ such that $q(n, \epsilon, \tau) \geq 2^{n^c} \epsilon$ for some $0 < c < 1/2$.
\end{theorem}

We are not aware of any prior work on the hardness of agnostically learning a sigmoid with respect to Gaussian marginals. 


For the case of halfspaces, a result of Kalai et al.~\cite{kalai2008agnostically} showed that any halfspace can be agnostically learned with respect to  Gaussian marginals in time and sample complexity $n^{O(1/\eps^4)}$, which was later improved to $n^{O(1/\eps^2)}$ \cite{diakonikolas2010bounded}.  The only known hardness result for this problem is due to Klivans and Kothari \cite{klivans2014embedding} who gave a quasipolynomial lower bound based on the hardness of learning sparse parity with noise.  Here we give the first exponential lower bound:
\begin{theorem}
	Consider the above setup with $\cH_{\hs}$, the class of unit-weight halfspaces on $\R^n$. For any $\epsilon$, there exists $\tau = n^{-\Theta(1/\epsilon)}$ such that $q(n, \epsilon, \tau) \geq 2^{n^c} \epsilon$ for some fixed constant $0 < c < 1/2$.
\end{theorem}

Since it takes $\Theta(1/\tau^2)$ samples to simulate a query of tolerance $\tau$, our constraint on $\tau$ here can be interpreted as saying that to avoid the exponential query lower bound, one needs sample complexity at least $\Theta(1/\tau^2) = n^{\Theta(1/\epsilon)}$, nearly matching the upper bound of \cite{kalai2008agnostically,diakonikolas2010bounded}.

These results are formally stated and proved in \cref{sec:lower-bounds}. More generally, we show in \cref{sec:non-polynomial} that our results give superpolynomial SQ lower bounds for agnostically learning any non-polynomial activation. (See \cref{app:sq-subtleties} for some discussion of subtleties in interpreting these bounds.)

A notable property of our lower bounds is that they hold for {\em general} statistical queries. As noted by several authors \cite{andoni2014learning,vempala2019gradient}, proving SQ lower bounds for real-valued learning problems often requires further restrictions on the types of queries the learner is allowed to make (e.g., correlational or Lipschitz queries). 

Another consequence of our framework is the first SQ lower bound for agnostically learning monomials with respect to Gaussian marginals.  In contrast, for the realizable (noiseless) setting, recent work due to Andoni et al.~\cite{andoni2019attribute} gave an attribute-efficient SQ algorithm for learning monomials. They left open the problem of making their results noise-tolerant.  We show in \cref{sec:monomials} that in the agnostic setting, no efficient SQ algorithm exists.

\begin{theorem}\label{thm:monomials-thm}
    Consider the above setup with $\cH_{\mon}$, the class of multilinear monomials of degree at most $d$ on $\R^n$. For any $\epsilon \leq \exp(-\Theta(d))$ and $\tau \leq \epsilon^2$, $q(n, \epsilon, \tau) \geq n^{\Theta(d)} \tau^{5/2}$.
\end{theorem}

\paragraph{Our Approach}
Our approach deviates from the standard template for proving SQ lower bounds and may be of independent interest.  In almost all prior work, SQ lower bounds are derived by constructing a sufficiently large family of nearly orthogonal functions with respect to the underlying marginal distribution. Instead, we will use a reduction-based approach:

\begin{itemize}
    \item We show that an algorithm for agnostically learning a single nonlinear activation $\phi$ can be used as a subroutine for learning depth-two neural networks of the form $\psi(\sum_{i} \phi(w^i \cdot x))$ where $\psi$ is any monotone, Lipschitz activation.  This reduction involves an application of functional gradient descent via the Frank--Wolfe method with respect to a (nonstandard) convex surrogate loss.  

  \item We apply recent work due to \cite{diakonikolas2020algorithms} and \cite{goel2020superpolynomial} that gives SQ lower bounds for learning depth-two neural networks of the above form in the probabilistic concept model. For technical reasons, our lower bound depends on the norms of these depth-two networks, and we explicitly calculate them for ReLU and sigmoid. 

 \item We prove that the above reduction can be performed using only statistical queries.  To do so, we make use of some subtle properties of the surrogate loss and the functional gradient method itself. 

\end{itemize}

Our reduction implies the following new relationship between two well-studied models of learning: if concept class ${\cal C}$ is efficiently agnostically learnable, then the class of monotone, Lipschitz functions of linear combinations of ${\cal C}$ is learnable in the {\em probabilistic concept} model due to Kearns and Schapire \cite{kearns1994efficient}. We cannot hope to further strengthen the conclusion to {\em agnostic} learnability of monotone, Lipschitz functions of combinations of ${\cal C}$: the concept class of literals {\em is} agnostically learnable, but we show exponential SQ lower bounds for agnostically learning the class of majorities of literals, i.e., halfspaces (see also \cite{klivans2014embedding}).

\paragraph{Related Work}
Several recent papers have considered the computational complexity of learning simple neural networks \cite{bach2017breaking,goel2017reliably,yehudai2020learning,frei2020agnostic,  klivans2014embedding,livni2014computational,song2017complexity,vempala2019gradient, goel2019time,goel2020superpolynomial,diakonikolas2020algorithms}.  The above works either consider one-layer neural networks (as opposed to learning single neurons), or make use of discrete distributions (rather than Gaussian marginals), or hold for narrower classes of algorithms (rather than SQ algorithms). Goel et al.\ \cite{goel2019time} give a quasipolynomial correlational SQ lower bound for proper agnostic learning of ReLUs with respect to Gaussian marginals. They additionally give a similar computational lower bound assuming the hardness of learning sparse parity with noise.

The idea of using functional gradient descent to learn one hidden layer neural networks appears in work due to Bach \cite{bach2017breaking}, who considered an ``incremental conditional gradient algorithm'' that at each iteration implicitly requires an agnostic learner to complete a ``Frank--Wolfe step.''  A key idea in our work is to optimize with respect to a particular convex functional (surrogate loss) in order to obtain SQ learnability for depth-two neural networks {\em with a nonlinear output activation.} We can then leverage SQ lower bounds for this broader class of neural networks.

Functional gradient descent or gradient boosting methods have been used frequently in learning theory, especially in online learning (see e.g., \cite{friedman2001greedy, mason2000boosting, schapire2012boosting, beygelzimer2015online, hazan2016introduction}.) 

For Boolean functions, the idea to use boosting to learn majorities of a base class appeared in Jackson~\cite{Jackson}, who boosted a weak parity learning algorithm in order to learn thresholds of parities (TOP).  Agnostic, distribution-specific boosting algorithms for Boolean functions have appeared in works due to Kalai and Kanade \cite{kanade2009potential} and also Feldman \cite{feldman2009distribution}. Agnostic boosting in the context of the SQ model is explored in \cite{feldman2012complete}, where an SQ lower bound is given for agnostically learning monotone conjunctions with respect to the uniform distribution on the Boolean hypercube.

The SQ lower bounds we obtain for agnostically learning halfspaces can be derived using one of the above boosting algorithms due to Kalai and Kanade \cite{kanade2009potential} or Feldman \cite{feldman2009distribution} in place of functional gradient descent, as halfspaces are Boolean functions.



\paragraph{Independent Work}
Independently and concurrently, Diakonikolas et al.\ \cite{diakonikolas2020nearoptimal} have obtained similar results for agnostically learning halfspaces and ReLUs. Rather than using a reduction-based approach, they construct a hard family of Boolean functions. They show that an agnostic learner for halfspaces or ReLUs would yield a learner for this family, which would solve a hard unsupervised distribution-learning problem considered in \cite{diakonikolas2017statistical}. Quantitatively, the lower bound they obtain is that agnostic learning of halfspaces or ReLUs up to excess error $\epsilon$ using queries of tolerance $n^{-\poly(1/\epsilon)}$ requires at least $n^{\poly(1/\epsilon)}$ queries. These results are technically incomparable with ours. For queries of similar tolerance, our bound of $2^{n^c} \epsilon$ scales exponentially with $n$ whereas theirs only scales polynomially, so that for any constant $\epsilon$ our bound is exponentially stronger. But our bound does not scale directly with $1/\epsilon$ (other than via the induced constraint on tolerance, which does scale as $n^{-\poly(1/\epsilon)})$. Our work also extends to general non-polynomial activations, while theirs does not.

\paragraph{Organization}
We cover the essential definitions, models and existing lower bounds that we need in the preliminaries. Our main reduction, which says that if we could agnostically learn a single neuron, then we could learn depth-two neural networks composed of such neurons, is set up as follows. In \cref{sec:func-gd} we explain our usage of functional gradient descent, with \cref{as:base-learner} formally stating the kind of agnostic learning guarantee we require for a single neuron. The main reduction itself is \cref{thm:boosting-surloss}, the subject of \cref{sec:func-gd-surr}. In \cref{sec:lower-bounds,,sec:monomials,,sec:non-polynomial} we derive the formal lower bounds which follow as a consequence of our reduction. Finally in \cref{sec:upper-bounds}, we contrast these lower bounds by also including some simple upper bounds.

\section{Preliminaries}
\paragraph{Notation}
Let $D$ be a distribution over $\R^n$, which for us will be the standard Gaussian $\cN(0, I_n)$ throughout. We will work with the $L^2$ space $L^2(\R^n, D)$ of functions from $\R^n$ to $\R$, with the inner product given by $\inn{f, g}_D = \Ex_{D}[fg]$. The corresponding norm is $\norm{f}_D = \sqrt{\Ex_D [f^2]}$. We refer to the ball of radius $R$ as $\cB_D(R) = \{ f \in L^2(\R^n, D) \mid \|f\|_D \leq R \}$. We will omit the subscripts when the meaning is clear from context. Given vectors $u, v \in \R^n$, we will refer to their Euclidean dot product by $\dotp{u}{v}$ and the Euclidean norm by $\|u\|_2$. Given a function $\ell(a, b)$ we denote its partial derivative with respect to its first parameter, $\frac{\partial \ell}{\partial a}(a, b)$, by $\partial_1 \ell(a, b)$.

A Boolean probabilistic concept, or $p$-concept, is a function that maps each point $x$ to a random $\{\pm 1\}$-valued label $y$ in such a way that $\Ex[y|x] = f^*(x)$ for a fixed function $f^* : \R^n \to [-1, 1]$, known as its conditional mean function. We will use $D_{f^*}$ to refer to the (unique) induced labeled distribution on $\R^n \times \{\pm 1\}$, i.e.\ we say $(x, y) \sim D_{f^*}$ if the marginal distribution of $x$ is $D$ and $\Ex[y | x] = f^{*}(x)$. We also sometimes use $y \sim f^*(x)$ to say that $y \in \{\pm 1\}$ and $\Ex[y | x] = f^{*}(x)$.

\paragraph{Statistical Query (SQ) Model}
A statistical query is specified by a query function $\phi : \R^n \times \R \to [-1, 1]$.  Given a labeled distribution $\cD$ on $\R^n \times \R$, the SQ model allows access to an SQ oracle (known as the STAT oracle in the SQ literature) that accepts a query $\phi$ of specified tolerance $\tau$, and responds with a value in $[\Ex_{(x,y) \sim \cD}[\phi(x,y)] - \tau, \Ex_{(x,y) \sim \cD}[\phi(x,y)] + \tau]$. One can interpret the tolerance $\tau$ as capturing the notion of sample complexity in traditional PAC algorithms. Specifically, it takes $\Theta(1/\tau^2)$ samples to simulate a query of tolerance $\tau$, and this is sometimes referred to as the estimation complexity of an SQ algorithm.

Let $\cC$ be a class of Boolean $p$-concepts over $\R^n$, and let $D$ be a distribution on $\R^n$. We say that a learner learns $\cC$ with respect to $D$ up to $L^2$ error $\epsilon$ if, given only SQ oracle access to $D_{f^*}$ for some unknown $f^* \in \cC$, and using arbitrary queries, it is able to output $f : \R^n \to [-1, 1]$ such that $\norm{f - f^*}_D \leq \epsilon$. It is worth emphasizing that a query to $D_{f^*}$ takes in a Boolean rather than a real-valued label, i.e.\ is really of the form $\phi : \R^n \times \{\pm 1\} \to [-1, 1]$. In contrast, a query to a generic distribution $\cD$ on $\R^n \times \R$ takes in real-valued labels, and in \cref{as:base-learner} we define a form of learning that operates in this more generic setting.


One of the chief features of the SQ model is that one can give strong information theoretic lower bounds on learning a class $\cC$ in terms of its so-called statistical dimension.
\begin{definition}
	Let $D$ be a distribution on $\R^n$, and let $\cC$ be a  real-valued or Boolean concept class on $\R^n$. The  \emph{average (un-normalized) correlation} of $\cC$ is defined to be $ \rho_D(\cC) = \frac{1}{|\cC|^2} \sum_{c, c' \in \cC} |\inn{c, c'}_D|$.
	The \emph{statistical dimension on average} at threshold $\gamma$, $\sda_D(\cC, \gamma)$, is the largest $d$ such that for all $\cC' \subseteq \cC$ with $|\cC'| \geq |\cC|/d$, $\rho_D(\cC') \leq \gamma$.
\end{definition}

In the $p$-concept setting, lower bounds against general queries in terms of SDA were first formally shown in \cite{goel2020superpolynomial}.
\begin{theorem}[\cite{goel2020superpolynomial}, Cor.\ 4.6]\label{thm:sq-lower-bound}
	Let $D$ be a distribution on $\R^n$, and let $\cC$ be a $p$-concept class on $\R^n$. Say our queries are of tolerance $\tau$, the final desired $L^2$ error is $\epsilon$, and that the functions in $\cC$ satisfy $\|f^*\| \geq \beta$ for all $f^* \in \cC$. For technical reasons, we will require $\tau \leq \epsilon^2$,  $\epsilon \leq \beta/3$ (see \cref{app:sq-subtleties} for some discussion). Then learning $\cC$ up to $L^2$ error $\epsilon$ (we may pick $\epsilon$ as large as $\beta/3$) requires at least $\sda_D(\cC, \tau^2)$ queries of tolerance $\tau$.
\end{theorem}

A recent result of Diakonikolas et al \cite{diakonikolas2020algorithms} gave the following construction of one-layer neural networks on $\R^n$ with $k$ hidden units, i.e.\ functions of the form $g(x) = \psi ( \sum_{i=1}^k a_i\phi(\dotp{x}{w_i}) )$ for activation functions $\psi, \phi : \R \to \R$ and weights $w_i \in \R^n, a_i \in \R$.
\begin{theorem}[\cite{diakonikolas2020algorithms}]\label{thm:ilias-construction}
	There exists a class $\cG$ of one-layer neural networks on $\R^n$ with $k$ hidden units such that for some universal constant $0 < c < 1/2$ and $\gamma = n^{\Theta(k(c - 1/2))}$, $\sda(\cG, \gamma) \geq 2^{n^c}$. This holds for any $\psi : \R \to [-1, 1]$ that is odd, and $\phi \in L^2(\R, \cN(0, 1))$ that has a nonzero Hermite coefficient of degree greater than $k/2$. Further, the weights satisfy $|a_i| = 1/k$ and $\|w_i\|_2 = 1$ for all $i$.
\end{theorem}

We will be interested in the following special cases. Full details of the construction and proofs of the norm lower bounds are in \cref{app:norm-calculation}.
\begin{corollary}\label{cor:ilias-instantiations}
    For the following instantiations of $\cG$, with accompanying norm lower bound $\beta$ (i.e.\ such that $\norm{g} \geq \beta$ for all $g \in \cG$), there exist $\tau = n^{-\Theta(k)}$ and $\epsilon \geq \tau$ such that learning $\cG$ up to $L^2$ error $\epsilon$ requires at least $2^{n^c}$ queries of tolerance $\tau$, for some $0 < c < 1/2$. \begin{enumerate}[(a)]
        \item ReLU nets: $\psi = \tanh$, $\phi = \relu$. Then $\beta = \Omega(1/k^6)$ (\cref{lem:f-norm-lower-bound-relu}), so we may take $\epsilon = \Theta(1/k^6)$.
        \item Sigmoid nets: $\psi = \tanh$, $\phi = \sigma$. Then $\beta = \exp(-O(\sqrt{k}))$ (\cref{lem:f-norm-lower-bound-sigmoid}), so we may take $\epsilon = \exp(-\Theta(\sqrt{k}))$.
        \item Majority of halfspaces: $\psi = \phi = \sgn$.  Being Boolean functions, here $\beta = 1$ exactly, so we may take $\epsilon = \Theta(1)$.
    \end{enumerate}
\end{corollary}

\paragraph{Convex Optimization Basics}
Over a general inner product space $\cZ$, a function $\cfunc : \cZ \to \R$ is convex if for all $\alpha \in [0,1]$ and $z, z' \in \cZ$, $\cfunc(\alpha z + (1 - \alpha)z') \leq \alpha \cfunc(z) + (1 - \alpha)\cfunc(z')$. We say that $s \in \cZ$ is a subgradient of $\cfunc$ at $z$ if $\cfunc(z + h) - \cfunc(z) \geq \inn{s, h}$. We say that $\cfunc$ is $\beta$-smoothly convex if for all $z, h \in \cZ$ and any subgradient $s$ of $\cfunc$ at $z$, \[  \cfunc(z + h) - \cfunc(z) - \inn{s, h} \leq \frac{\beta}{2}\norm{h}^2. \] If there is a unique subgradient of $\cfunc$ at $z$, we simply refer to it as the gradient $\nabla \cfunc(z)$. It is easily proven that smoothly convex functions have unique subgradients at all points. Another standard property is the following: for any $z, z' \in \cZ$, \begin{equation}
\cfunc(z) - \cfunc(z') \leq \inn{\nabla \cfunc(z), z - z'} - \frac{1}{2\beta} \norm{\nabla \cfunc(z) - \nabla \cfunc(z')}^2. \label{eq:smoothness-property}
\end{equation}

In this paper we will be concerned with convex optimization using the Frank--Wolfe variant of gradient descent, also known as conditional gradient descent. In order to eventually apply this framework to improper learning, we will consider a slight generalization of the standard setup. Let $\cZ' \subset \cZ$ both be compact, convex subsets of our generic inner product space. Say we have a $\beta$-smoothly convex function $\cfunc : \cZ \to \R$, and we want to solve $\min_{z \in \cZ'} \cfunc(z)$, i.e.\ optimize over the smaller domain, while allowing ourselves the freedom of finding subgradients that lie in the larger $\cZ$. The Frank--Wolfe algorithm in this ``improper'' setting is \cref{alg:fw-generic}.
\begin{algorithm}
	\caption{Frank--Wolfe gradient descent over a generic inner product space}
	\label{alg:fw-generic}
	\begin{algorithmic}
		\State Start with an arbitrary $z_0 \in \cZ$.
		\For{$t = 0, \dots, T$} 
		\State Let $\gamma_t = \frac{2}{t+2}$.
		\State Find $s \in \cZ$ such that $ \inn{s, -\nabla \cfunc(z_t) } \geq \max_{s' \in \cZ'}\inn{s', -\nabla \cfunc(z_t)} - \frac{1}{2} \delta \gamma_t C_\cfunc$. 
		\State Let $z_{t+1} = (1 - \gamma_t) z_t + \gamma_t s$.
		\EndFor
	\end{algorithmic}
\end{algorithm}

The following theorem holds by standard analysis (see e.g.\ \cite{jaggi2013revisiting}). For convenience, we provide a self-contained proof in \cref{app:fw-proof}.

\begin{theorem}\label{thm:fw-guarantee}
	Let $\cZ' \subseteq \cZ$ be convex sets, and let $\cfunc : \cZ \to \R$ be a $\beta$-smoothly convex function. Let $C_\cfunc = \beta \diam(\cZ)^2$.  For every $t$, the iterates of \cref{alg:fw-generic} satisfy \[ \cfunc(z_t) - \min_{z' \in \cZ'} \cfunc(z') \leq \frac{2C_\cfunc}{t + 2}(1 + \delta). \]
\end{theorem}

\section{Functional gradient descent}\label{sec:func-gd}
Let $\ell : \R \times \R \to \R$ be a loss function. Given a $p$-concept $f^*$ and its corresponding labeled distribution $D_{f^*}$, the population loss of a function $f : \R^n \to \R$ is given by $L(f) = \Ex_{(x, y) \sim D_{f^*}}[\ell(f(x), y)]$. We will view $L$ as a mapping from $L^2(\R^n, D)$ to $\R$, and refer to it as the loss functional. The general idea of functional gradient descent is to try to find an $f$ in a class of functions $\cF$ that minimizes $L(f)$ by performing gradient descent in function space. When using Frank--Wolfe gradient descent, the key step in every iteration is to find the vector that has the greatest projection along the negative gradient, which amounts to solving a linear optimization problem over the domain. When $\cF$ is the convex hull $\conv(\cH)$ of a simpler class $\cH$, this can be done using a sufficiently powerful agnostic learning primitive for $\cH$. Thus we can ``boost'' such a primitive in a black-box manner to minimize $L(f)$.

Let $\cH \subset L^2(\R^n, D)$ be a base hypothesis class for which we have an agnostic learner with the following guarantee:

\begin{assumption}\label{as:base-learner}
	There is an SQ learner for $\cH$ with the following guarantee. Let $\cD$ be any labeled distribution on $\R^n \times \R$ such that the marginal on $\R^n$ is $D = \cN(0, I_n)$. Given only SQ access to $\cD$, the learner outputs a function $f \in \cB(\diam(\cH)/2)$ such that \[ \Ex_{(x, y) \sim \cD}[f(x) y] \geq \max_{h \in \cH} \Ex_{(x, y) \sim \cD}[h(x) y] - \epsilon \] using $q(n, \epsilon, \tau)$ queries of tolerance $\tau$.
	
	Notice that we do \emph{not} require $f$ to lie in $\cH$, i.e.\ the learner is allowed to be improper, but we do require it to have norm at most $\diam(\cH)/2$. This is to make the competitive guarantee against $\cH$ meaningful, since otherwise the correlation can be made to scale arbitrarily with the norm.
\end{assumption}

With such an $\cH$ in place, we define $\cF = \conv(\cH)$. We assume that $f^* \in \cF$. Our objective will be to agnostically learn $\cF$: to solve $\min_{f \in \cF} L(f)$ in such a way that $L(f) - L(f^*) \leq \epsilon$.

To be able to use Frank--Wolfe, we require some assumptions on the loss function $\ell$.
\begin{assumption}\label{as:convex-loss}
	The loss function $\ell : \R \times \R \to \R$ is $\beta$-smoothly convex in its first parameter.
\end{assumption}


From this assumption, orresponding properties of the loss functional $L$ now follow. First we establish the subgradient, which will itself be an element of $L^2(\R^n, D)$, i.e.\ a function from $\R^n$ to $\R$. Let $f, h : \R^n \to \R$. Observe that at for every $x \in \R^n, y \in \R$, the subgradient property of $\ell$ tells us that \[ \ell(f(x) + h(x), y) - \ell(f(x), y) \geq \partial_1 \ell(f(x), y) h(x). \] Taking expectations over $(x, y) \sim D_{f^*}$, this yields \begin{align*}
L(f + h) - L(f) &\geq \Ex_{(x, y) \sim D_{f^*}}[\partial_1 \ell(f(x), y) h(x)] \\
&= \Ex_{x \sim D}[\Ex_{y | x}[\partial_1 \ell(f(x), y)] h(x)] \\
&= \inn{s, h},
\end{align*} where \[ s : x \mapsto \Ex_{y | x}[\partial_1 \ell(f(x), y)] = \Ex_{y \sim f^*(x)}[\partial_1 \ell(f(x), y)] \] is thus a subgradient of $L$ at $f$. $\beta$-smooth convexity is also easily established. Taking expectations over $(x, y) \sim D_{f^*}$ of the inequality \[ \ell(f(x) + h(x), y) - \ell(f(x), y) - \partial_1 \ell(f(x), y) h(x) \leq \frac{\beta}{2}h(x)^2, \] we get \[
L(f+h) - L(f) - \inn{s, h} \leq \frac{\beta}{2}\|h\|^2
\] for the same subgradient $s$. By smooth convexity, this subgradient is unique and so we can say that the gradient of $L$ at $f$ is given by $\nabla L(f) : x \mapsto \Ex_{y \sim f^*(x)}[\partial_1 \ell(f(x), y)]$.

\begin{example}\label{ex:sq-loss}
	The canonical example is the squared loss functional, with $\ell_\sq(a, b) = (a - b)^2$, which is 2-smoothly convex. Here the gradient has a very simple form, since $\partial_1 \ell_\sq(a, b) = 2(a - b)$, and so \[ \Ex_{y \sim f^*(x)}[\partial_1 \ell_\sq(f(x), y)] = \Ex_{y \sim f^*(x)}[2(f(x) - y)] = 2(f(x) - f^*(x)), \] i.e.\ $\nabla L_\sq(f) = 2(f - f^*)$. In fact, it is easily calculated that \begin{align*} L_\sq(f) = \Ex_{(x, y) \sim D_{f^*}}[(f(x) - y)^2] &= \Ex_{(x, y) \sim D_{f^*}}[f(x)^2] - 2 \Ex_{(x, y) \sim D_{f^*}}[f(x) y] + \Ex_{(x, y) \sim D_{f^*}}[y^2] \\
	&= \Ex_{x \sim D}[f(x)^2] - 2\Ex_{x \sim D}[f(x) \Ex[y | x]] + \Ex_{(x, y) \sim D_{f^*}}[y^2]  \\
	&= \|f\|^2 - 2 \inn{f, f^*} + 1,
	\end{align*}
	It is also useful to note that \begin{equation}
	L_\sq(f) - L_\sq(f^*) = \norm{f - f^*}^2. \label{eq:sqloss-diff}
	\end{equation}
\end{example}

\paragraph{Frank--Wolfe using statistical queries}\label{par:fw-sq}
We see that our loss functional is a $\beta$-smoothly convex functional on the space $L^2(\R^n, D)$. We can now use Frank--Wolfe if we can solve its main subproblem: finding an approximate solution to $\max_{h \in \cF} \inn{h, - \nabla L(f)}$, where $f$ is the current hypothesis during some iteration. Since this is a linear optimization objective and $\cF = \conv(\cH)$, this is the same as solving $\max_{h \in \cH} \inn{h, - \nabla L(f)}$. This is almost the guarantee that \cref{as:base-learner} gives us, but some care is in order. What we have SQ access to is the labeled distribution $D_{f^*}$ on $\R^n \times \{\pm 1\}$. It is not clear that we can rewrite the optimization objective in such a way that \begin{equation}\label{eq:fw-subproblem} \max_{h \in \cH} \Ex_{x \sim D}[-h(x) \nabla L(f)(x)] = \max_{h \in \cH} \Ex_{(x, y') \sim \cD} [h(x) y'] \end{equation} for some distribution $\cD$ on $\R^n \times \R$ \emph{that we can simulate SQ access to}. Naively, we might try to do this by letting $\cD$ be the distribution of $(x, -\nabla L(f)(x))$ for $x \sim D$, so that a query $\phi : \R \times \R \to \R$ to $\cD$ can be answered with $ \Ex_{(x, y') \sim \cD}[\phi(x, y')] = \Ex_{x \sim D}[\phi(x, -\nabla L(f)(x))]$. But the issue is that in general $\nabla L(f)(x)$ will depend on $f^*(x)$, which we do not know --- all we have access to is $D_{f^*}$.

It turns out that for the loss functions we are interested in, we can indeed find a suitable such $\cD$. We turn to the details now.

\section{Functional gradient descent guarantees on surrogate loss}\label{sec:func-gd-surr}
The functional GD approach applied directly to squared loss would allow us to learn $\cF = \conv(\cH)$ using a learner for $\cH$ (that satisfied \cref{as:base-learner}). But by considering a certain surrogate loss, we can use the same learner to actually learn $\psi \circ \cF = \{\psi \circ f \mid f \in \cF\}$ for an outer activation function $\psi$. This is particularly useful as we can now capture $p$-concepts corresponding to functions in $\cF$ by using a suitable $\psi : \R \to [-1, 1]$. For example, the common softmax activation corresponds to taking $\psi = \tanh$.

Assume that $\Ex[y|x] = \psi(f^*(x))$ for some activation $\psi: \R \to \R$ which is non-decreasing and $\lambda$-Lipschitz. Instead of the squared loss, we will consider the following surrogate loss:
\[ \lsur(a, b) = \int_0^a (\psi(u) - b) du. \]
It is not hard to see that $\lsur(a, b)$ is convex in its first parameter due to the non-decreasing property of $\psi$, and that $\partial_1 \lsur(a, b) = \psi(a) - b$. In fact it is $\lambda$-smoothly convex:
\begin{align*}
&\lsur(a + t , b) - \lsur(a , b) - \partial_1 \lsur(a, b)t\\
&= \int_0^{a + t} (\psi(u) - b) du - \int_0^a (\psi(u) - b) du - (\psi(a) - b)t\\
&= \int_a^{a + t} (\psi(u) - b) du - (\psi(a) - b)t\\
&= \int_a^{a + t} (\psi(u) - \psi(a)) du\\
&\le \int_a^{a + t} \lambda (u - a) du \\
&= \frac{\lambda t^2}{2}.
\end{align*}

The gradient of the surrogate loss functional, $\Lsur(f) = \Ex_{(x, y) \sim D_{\psi \circ f^*}}[\lsur(f(x), y)]$, is given by \[ \nabla \Lsur(f) : x \mapsto \Ex_{y \sim \psi(f^*(x))} [\partial_1 \lsur(f(x), y)] = \psi(f(x)) - \psi(f^*(x)), \] i.e.\ $\nabla \Lsur(f) = \psi \circ f - \psi \circ f^*$.

We still need to show that the Frank--Wolfe subproblem can be solved using access to just $D_{\psi \circ f^*}$. Observe that \begin{align*}
\Ex_{x \sim D}[-h(x) \nabla \Lsur(f)(x)] &= \Ex_{x \sim D}\left[h(x)(\psi(f^*(x)) - \psi(f(x)))\right] \\
&= \Ex_{x \sim D}\left[h(x)\left(\Ex_{y \sim \psi(f^*(x))}[y] - \psi(f(x))\right)\right] \\
&= \Ex_{(x, y) \sim D_{\psi \circ f^*}}[h(x)(y - \psi(f(x)))] \\
&= \Ex_{(x, y') \sim \cD} [h(x) y'],
\end{align*} where $\cD$ is the distribution of $(x, y - \psi(f(x)))$ for $(x, y) \sim D_{\psi \circ f^*}$. We can easily simulate SQ access to this using $D_{\psi \circ f^*}$: if $\phi$ is any query to $\cD$, then \begin{equation}\label{eq:rewriting-fw} \Ex_{(x, y') \sim \cD}[\phi(x, y')] = \Ex_{(x, y) \sim D_{\psi\circ f^*}}[\phi(x, y - \psi(f(x)))] = \Ex_{(x, y) \sim D_{\psi \circ f^*}} [\phi'(x, y)] \end{equation} for the modified query $\phi'(x, y) = \phi(x, y - \psi(f(x)))$. This means we can rewrite the optimization objective to fit the form in \cref{eq:fw-subproblem}. Thus for our surrogate loss, \cref{as:base-learner} allows us to solve the Frank--Wolfe subproblem, giving us \cref{alg:fw-learner-surloss} for learning $\cF$.

\begin{algorithm}
	\caption{Frank--Wolfe for solving $\min_{f \in \cF} \Lsur(f)$}
	\label{alg:fw-learner-surloss}
	\begin{algorithmic}
		\State Start with an arbitrary $f_0 \in \cB(\diam(\cH)/2)$.
		\For{$t = 0, \dots, T$}
		\State Let $\gamma_t$ be $\frac{2}{t + 2}$.
		\State Let $\cD_t$ be the distribution of $(x, y - \psi(f_t(x)))$ for $(x, y) \sim D_{\psi \circ f^*}$.
		\State Using \cref{as:base-learner}, find $h \in \cB(\diam(\cH)/2)$ such that \[ \Ex_{(x, y') \sim \cD_t}[h(x)y'] \geq \max_{h' \in \cH} \Ex_{(x, y') \sim \cD_t} [h'(x) y'] - \frac{1}{2}\gamma_t \lambda \diam(\cH)^2 \]
		\State Let $f_{t+1} = (1 - \gamma_t)f_t + \gamma_t h$.
		\EndFor
	\end{algorithmic}
\end{algorithm}

\begin{theorem}\label{thm:boosting-surloss}
	Let $\cH$ be a class for which \cref{as:base-learner} holds, and let $\cF = \conv(\cH)$. Given SQ access to $D_{\psi \circ f^*}$ for a known non-decreasing $\lambda$-Lipschitz activation $\psi$ and an unknown $f^* \in \cF$, suppose we wish to learn $\psi \circ f^*$ in terms of surrogate loss, i.e.\ to minimize $\Lsur(f)$. Then after $T$ iterations of \cref{alg:fw-learner-surloss}, we have the following guarantee: \[ \Lsur(f_T) - \Lsur(f^*) \leq \frac{4\lambda \diam(\cH)^2}{T + 2}. \] In particular, we can achieve $\Lsur(f_T) - \Lsur(f^*) \leq \epsilon$ after $T = O(\frac{\lambda \diam(\cH)^2}{\epsilon})$ iterations. Assuming our queries are of tolerance $\tau$, the total number of queries used is at most $T q(n, \epsilon/4, \tau) = O(\frac{\lambda \diam(\cH)^2}{\epsilon} q(n, \epsilon/4, \tau))$.
\end{theorem}
\begin{proof}
	By the preceding discussion, the surrogate loss functional is $\lambda$-smoothly convex, and \cref{alg:fw-learner-surloss} is a valid special case of \cref{alg:fw-generic}, with $\cZ = \cB(\diam(\cH)/2)$ and $\cZ' = \conv(\cF)$. Thus the guarantee follows directly from \cref{thm:fw-guarantee} (setting $\delta = 1$).
	
	To bound the number of queries, observe that it is sufficient to run for $T = \frac{4\lambda \diam(\cH)^2}{\epsilon} - 2$ rounds. In the $t\th$ iteration, we invoke \cref{as:base-learner} with \[ \epsilon' = \frac{1}{2}\gamma_t \lambda \diam(\cH)^2 = \frac{\lambda \diam(\cH)^2}{t + 2} \geq \frac{\lambda \diam(\cH)^2}{T + 2} = \frac{\epsilon}{4}. \] Since $q(n, \epsilon', \tau) \leq q(n, \epsilon/4, \tau)$, the bound follows.
\end{proof}

Lastly, we can show that minimizing surrogate loss also minimizes the squared loss. Observe first that $\nabla \Lsur(f^*) = 0$. Thus, applying \cref{eq:smoothness-property} with $z = f^*$ and $z' = f$, we obtain \begin{align}
\Lsur(f) - \Lsur(f^*) &\geq \frac{1}{2\lambda}\norm{\nabla \Lsur(f) - \nabla \Lsur(f^*)}^2 \nonumber \\
&= \frac{1}{2\lambda}\norm{\psi \circ f - \psi \circ f^*}^2 \label{eq:surr-l2dist-squared} \\
&= \frac{1}{2\lambda}(L_\sq(\psi \circ f) - L_\sq(\psi \circ f^*)), \nonumber
\end{align} where $L_\sq$ is squared loss w.r.t. $D_{\psi \circ f^*}$ and the last equality is \cref{eq:sqloss-diff}. In particular, \cref{eq:surr-l2dist-squared} implies that $\psi \circ f$ achieves the following $L^2$ error with respect to $\psi \circ f^*$: \begin{equation}\label{eq:surr-l2dist} \norm{\psi \circ f - \psi \circ f^*} \leq \sqrt{2\lambda \left(L_\sur(f) - L_\sur(f^*)\right)}. \end{equation}

\section{Lower bounds on learning ReLUs, sigmoids, and halfspaces}\label{sec:lower-bounds}
The machinery so far has shown that if we could agnostically learn a single unit (e.g.\ a ReLU or a sigmoid), we could learn depth-two neural networks composed of such units. Since we have lower bounds on the latter problem, this yields the following lower bounds on the former.

\begin{theorem}
	Let $\cH_{\relu} = \{x \mapsto \pm \relu(\dotp{w}{x}) \mid \|w\|_2 \leq 1 \}$ be the class of ReLUs on $\R^n$ with unit weight vectors.\footnote{We use $\pm \relu$ for simplicity. Any learner can handle this by doing a bit flip on its own.} Suppose that \cref{as:base-learner} holds for $\cH_{\relu}$. Then for any $\epsilon$, there exists $\tau = n^{-\Theta(\epsilon^{-1/12})}$ such that $q(n, \epsilon, \tau) \geq 2^{n^c} \epsilon$ for some $0 < c < 1/2$.
\end{theorem}
\begin{proof}
Since all our lower bound proofs are similar, to set a template we lay out all the steps as clearly as possible.
    \begin{itemize}
        \item Consider the class $\cG$ from \cref{thm:ilias-construction} instantiated with $\psi = \tanh$ (which is $1$-Lipschitz, so $\lambda = 1$) and $\phi = \relu$. By the conditions on the weights, we see that $\cG \subseteq \tanh \circ \cF_{\relu}$, where $\cF_{\relu} = \conv(\cH_{\relu})$. This construction has a free parameter $k$, which we will set based on $\epsilon$.
        \item By our main reduction (\cref{as:base-learner} and \cref{thm:boosting-surloss}), we can learn $\tanh \circ \cF_{\relu}$ with respect to $\Lsur$ up to agnostic error $\epsilon$ using $O(\frac{1}{\epsilon} q(n, \frac{\epsilon}{4}, \tau))$ queries of tolerance $\tau$. By \cref{eq:surr-l2dist}, this implies learning $\cG$ up to $L^2$ error $\sqrt{2\epsilon}$.
        \item We know that learning $\cG$ should be hard. Specifically, \cref{cor:ilias-instantiations}(a) states that if $\epsilon' = \Theta(1/k^6)$ and the queries are of tolerance $\tau = n^{-\Theta(k)}$, then learning up to $L^2$ error $\epsilon'$ should require $2^{n^c}$ queries.
        \item The loss our reduction achieves is $\epsilon' = \sqrt{2\epsilon}$, so we require $\sqrt{2\epsilon} \leq \Theta(1/k^6)$ for the bound to hold. Accordingly, we pick $k = \Theta(\epsilon^{-1/12})$, so that $\tau = n^{-\Theta(k)} = n^{-\Theta(\epsilon^{-1/12})}$.
        \item Thus we must have $\frac{1}{\epsilon} q(n, \frac{\epsilon}{4}, \tau) \geq 2^{n^c}$. Rearranging and rescaling $\epsilon$ gives the result.
    \end{itemize}
\end{proof}

\begin{theorem}
	Let $\cH_{\sigma} = \{x \mapsto \pm \sigma(\dotp{w}{x}) \mid \|w\|_2 \leq 1 \}$, where $\sigma$ is the standard sigmoid, be the class of sigmoid units on $\R^n$ with unit weight vectors. Suppose that \cref{as:base-learner} holds for $\cH_{\sigma}$. Then for any $\epsilon$, there exists $\tau = n^{-\Theta((\log 1/\epsilon)^2)}$ such that $q(n, \epsilon, \tau) \geq 2^{n^c} \epsilon$ for some $0 < c < 1/2$.
\end{theorem}
\begin{proof}
	Very similar to the above. We instantiate $\cG$ with $\psi = \tanh$, $\phi = \sigma$, and observe that $\cG \subseteq \tanh \circ \conv(\cH_\sigma)$ and that $\diam(\cH_\sigma) \leq 2$. In this case, \cref{cor:ilias-instantiations}(b) tells us that we require $\sqrt{2\epsilon} \leq e^{-\Theta(\sqrt{k})}$ for the lower bound to hold, so we pick $k = (\log 1/\epsilon)^2$. The result now follows exactly as before.
\end{proof}

We also obtain a lower bound on the class of halfspaces. The traditional way of phrasing agnostic learning for Boolean functions is in terms of the 0-1 loss, and it is not immediately obvious that the correlation loss guarantee of \cref{as:base-learner} is equivalent. But in \cref{app:boolean-loss}, we show that with a little care, they are indeed effectively equivalent. Note that for Boolean functions, functional GD is not essential; existing distribution-specific boosting methods \cite{kanade2009potential, feldman2009distribution} can also give us similar results here. 
\begin{theorem}
    Let $\cH_{\hs} = \{x \mapsto \sgn(\dotp{w}{x}) \mid \|w\|_2 \leq 1 \}$ be the class of halfspaces on $\R^n$ with unit weight vectors. Suppose that \cref{as:base-learner} holds for $\cH_{\hs}$. Then for any $\epsilon$, there exists $\tau = n^{-\Theta(1/\epsilon)}$ such that $q(n, \epsilon, \tau) \geq 2^{n^c} \epsilon^3$ for some $0 < c < 1/2$.
\end{theorem}
\begin{proof}
    To approximate the sign function using a Lipschitz function, we define $\lsgn(x)$ to be $-1$ for $x \leq -1/k$, $1$ for $x \geq 1/k$, and linearly interpolate in between. This function is $(k/2)$-Lipschitz. We claim that $\cG$ instantiated with $\psi = \phi = \sgn$ satisfies $\cG \subseteq \lsgn \circ \conv(\cH_\hs)$, with $\diam(\cG) = 2$. This is because as noted in \cref{thm:ilias-construction}, $\cG$ has weights $a_i \in \{\pm 1/k\}$, so the sum of halfspaces inside $\psi$ is always a multiple of $1/k$, and $\lsgn$ behaves the same as $\sgn$.
    
    \cref{thm:boosting-surloss} now lets us learn $\cG$ up to agnostic error $\epsilon$ (and hence $L^2$ error $\sqrt{2k\epsilon}$, by $\cref{eq:surr-l2dist}$) using $O(\frac{k^2}{\epsilon} q(n, \epsilon/4, \tau))$ queries of tolerance $\tau$. By \cref{cor:ilias-instantiations}(c), we only need $\sqrt{2k\epsilon} \leq \Theta(1)$ for the lower bound to hold, so we may take $k = \Theta(1/\epsilon)$ to get a lower bound of $2^{n^c}$. Thus $\frac{k^2}{\epsilon} q(n, \epsilon/4, \tau) \geq 2^{n^c}$, and rearrangement gives the result.
\end{proof}

\section{Lower bounds on learning general non-polynomial activations}\label{sec:non-polynomial}
Here we extend our lower bounds to general non-polynomial activations $\phi : \R \to \R$, by which we mean functions which have an infinite Hermite series $\phi = \sum_a \hat{\phi}_a H_a$, where the $H_a$ are the normalized probabilists' Hermite polynomials. We will again work with the class $\cG$ from \cref{thm:ilias-construction}, instantiated with this $\phi$ and $\psi = \tanh$.  In \cref{app:norm-calculation}, we define this construction formally, letting $g$ be the inner function and $f$ be $\psi \circ g$.

To apply our framework, we need a norm lower bound on $f$. In \cref{lem:norm-expression} we show that $\norm{g}$ is determined only by $k$, the number of hidden units (there $k = 2m$), and the Hermite expansion of $\phi$. The reason we require an infinite Hermite series for $\phi$ is so that this lower bound, viewed as a function of $k$, is nonzero for infinitely many $k$. This then implies that $f = \tanh \circ g$ must be nonzero for infinitely many $k$. Its norm can only possibly be a function of $\phi$ and $k$. In particular, we may assume that it satisfies a norm lower bound $\norm{f} \geq \beta(k)$, where $\beta$ is a function only of $k$ that is nonzero for infinitely many $k$. Here we view the dependence on $\phi$ as constant.

A few remarks are in order as to how such a bound $\beta(k)$ may be quantitatively established. If $\phi$ is either bounded or exhibits only polynomial growth, then the bound on $\norm{g}$ (\cref{lem:norm-expression}) gives a corresponding lower bound on $\norm{f}$ that is also purely a function of $k$. If $\phi$ is bounded, the calculation is straightforward and very similar to the $\phi = \sigma$ case (\cref{lem:f-norm-lower-bound-sigmoid}). If $\phi$ grows only like a polynomial, then one can use a truncation argument similar to the $\phi = \relu$ case (\cref{lem:f-norm-lower-bound-relu}).

By \cref{thm:sq-lower-bound} and \cref{cor:ilias-instantiations}, our lower bound of $2^{n^c}$ on learning $\cG$ holds for $\epsilon \leq \beta(k)/3$. Since we can pick $k$ as we like, let us say that for all sufficiently small $\epsilon$, we can achieve $\epsilon \leq \beta(k)/3$ by taking $k = k(\epsilon) = 3\beta^{-1}(\epsilon)$. The corresponding tolerance is then $\tau = n^{-\Theta(k(\epsilon))}$, which is still inverse superpolynomial in $n$.

We now get the following lower bound on learning $\cH = \{x \mapsto \phi(\dotp{w}{x}) \mid \|w\|_2 \leq 1\}$, again by the same arguments as in \cref{sec:lower-bounds}. We assume that $\norm{\phi} \leq R$ for some $R$, so that $\diam(\cH) \leq 2R$.
\begin{theorem}
    Suppose that \cref{as:base-learner} holds for $\cH$. Then for all sufficiently small $\epsilon$ and $\tau = n^{-\Theta(k(\epsilon))}$, $q(n, \epsilon, \tau) \geq 2^{n^c}\frac{\epsilon}{R^2}$ for some $0 < c < 1/2$.
\end{theorem}
\begin{proof}
    We have $\cG \subseteq \tanh \circ \conv(\cH)$. By functional GD wrt surrogate loss (\cref{thm:boosting-surloss}), we see that we can learn $\cG$ up to $L^2$ error $\sqrt{2\epsilon}$ using $O(\frac{R^2}{\epsilon}q(n, \epsilon, \tau))$ queries of tolerance $\tau$, but we must have  $O(\frac{R^2}{\epsilon}q(n, \epsilon, \tau)) \leq 2^{n^c}$.
\end{proof}

\section{Lower bounds on learning monomials}\label{sec:monomials}
In this section we show lower bounds against agnostically learning monomials with respect to the Gaussian, establishing \cref{thm:monomials-thm}. Let $\cH_\mon$ be the class of all multilinear monomials of total degree $d$ on $\R^n$. Clearly $|\cH_\mon| = \binom{n}{d} = n^{\Theta(d)}$. For any two distinct multilinear monomials $f, g$, clearly $\inn{f, g} = 0$ and moreover $\inn{\tanh \circ f, \tanh \circ g} = 0$ as well. Thus the class $\cG = \tanh \circ \cH_\mon$ consists entirely of orthogonal functions. By \cite[Lemma 2.6]{goel2020superpolynomial}, $\sda(\cG, \gamma) \geq |\cG| \gamma = n^{-\Theta(d)}\gamma$.

We still need a norm lower bound on $\cG$.
\begin{lemma}
    Let $x_S = \prod_{i \in S} x_i$ be an arbitrary degree-$d$ multilinear monomial on $\R^n$, where $S \subseteq [n]$ is a subset of size $d$. Then $\norm{\tanh \circ x_S} \geq \exp(-\Theta(d))$.
\end{lemma}
\begin{proof}
    Observe first that $\norm{x_S} = 1$. By Paley--Zygmund, we have \[ \Pr[x_S^2 \geq \theta \Ex[x_S^2]] \geq (1 - \theta)^2 \frac{\Ex[x_S^2]^2}{\Ex[x_S^4]}. \] By picking $\theta = 1/2$, say, and using the fact that by Gaussian hypercontractivity, \[ \frac{\Ex[x_S^2]^2}{\Ex[x_S^4]} = \prod_{i \in S} \frac{\Ex[x_i^2]^2}{\Ex[x_i^4]} \geq \exp(-\Theta(d)), \] we get that $\Pr[|x_S| \geq 1/2] \geq \exp(-\Theta(d))$.
    
    Now since $\tanh$ is monotonic and odd, we have \[ \Ex[\tanh(x_S)^2] \geq \tanh(1/2)^2 \Pr[|x_S| \geq 1/2] \geq \exp(-\Theta(d)). \]
\end{proof}

By \cref{thm:sq-lower-bound} with $\beta = \exp(-\Theta(d))$, we get that for any $\epsilon \leq \exp(-\Theta(d))$ and using queries of tolerance $\tau \leq \epsilon^2$, learning $\cG$ up to $L^2$ error $\epsilon$ takes at least $\sda(\cG, \tau^2) \geq n^{\Theta(d)}\tau^2$ queries.

Now we can use the same arguments as in \cref{sec:lower-bounds} to prove the following.
\begin{theorem}
    Suppose that \cref{as:base-learner} holds for $\cH_\mon$. Then for any $\epsilon \leq \exp(-\Theta(d))$ and $\tau \leq \epsilon^2$, $q(n, \epsilon, \tau) \geq n^{\Theta(d)}\tau^{5/2}$.
\end{theorem}
\begin{proof}
    Observe that $\cG \subseteq \tanh \circ \conv(\cH_\mon)$, and $\diam(\cH_\mon) \leq 2$. Using the surrogate loss with $\psi = \tanh$, \cref{as:base-learner} and \cref{thm:boosting-surloss} tell us that we can learn $\tanh \circ \conv(\cH_\mon)$ up to $L^2$ error $\sqrt{2\epsilon}$ (again by \cref{eq:sqloss-diff}) in $O(\frac{1}{\epsilon} q(n, \epsilon, \tau))$ queries of tolerance $\tau$. By our lower bound for $\cG$, we must have $\frac{1}{\epsilon} q(n, \epsilon, \tau) \geq n^{\Theta(d)}\tau^2$, or $q(n, \epsilon, \tau) \geq n^{\Theta(d)}\tau^{5/2}$ (since $\epsilon \geq \sqrt{\tau}$).
\end{proof}

\section{Upper bounds on learning ReLUs and sigmoids}\label{sec:upper-bounds}
We use a variant of the classic low-degree algorithm (\cite{linial1993constant}; see also \cite{kalai2008agnostically}) to provide simple upper bounds for agnostically learning ReLUs and sigmoids. With respect to $D = \cN(0, I_n)$, the $\delta$-approximate degree of a function $f : \R^n \to \R$ is the smallest $d$ such that there exists a degree-$d$ polynomial $p$ satisfying $\norm{f - p} \leq \delta$. We show that for any class of $\delta$-approximate degree $d$, picking $\delta = O(\epsilon)$ and simply estimating the Hermite coefficients of $x \mapsto \Ex[y|x]$ up to degree $d$ yields an agnostic learner up to error $\epsilon$, one that satisfies \cref{as:base-learner}. We assume bounded labels, say $y \in [-C, C]$ for some constant $C$.

Let $\cD$ be a distribution on $\R^n \times \R$ such that the marginal on $\R^n$ is $\cN(0, I_n)$. 
Let $f_\cmf(x) = \Ex[y|x]$ denote the conditional mean function of $\cD$, and note that $\|f_\cmf\| \leq C$. Observe that for any $f$, the correlation $\Ex_{(x, y) \sim \cD}[f(x)y]$ equals $\inn{f, f_\cmf}$. Let $\cH$ be a hypothesis class with $\delta$-approximate degree $d$ ($\delta$ to be determined), and let $R = \diam(\cH)/2$. Let $h_\opt \in \cH$ achieve $\max_{h \in \cH} \inn{h, f_\cmf}$.

Our algorithm will be based on approximating the low-degree Hermite coefficients of $f_\cmf$, which is equivalent to performing polynomial $L^2$ regression. It is well-known that in this context, where $d$ is the $\delta$-approximate degree, polynomial $L^1$ regression up to degree $d$ gives a squared loss guarantee of $\delta$ \cite{kalai2008agnostically}. But we will not be able to use this result directly since what we seek is a correlation guarantee. Instead, our approach will involve a sequence of inequalities relating the correlation achieved by $f_\cmf$, $h_\opt$, and their degree-$d$ approximations. A slight subtlety to keep in mind is that correlation can always be increased by scaling the function. This means that wherever scaling is possible, we have to take some care to rescale functions to have the maximum allowed norm, $R$.

Let $h^{\leq d}_\opt$ and $f_\cmf^{\leq d}$ be the Hermite components of degree at most $d$ of $h_\opt$ and $f_\cmf$ respectively. Let $\tilde{f}_\cmf^{\leq d} = \frac{R}{\|f_\cmf^{\leq d}\|}f_\cmf^{\leq d}$. Among polynomials of degree $d$ in $\cB(R)$, it is easy to see that $\tilde{f}_\cmf^{\leq d}$ maximizes $\inn{f, f_\cmf}$, so that \[ \inn{\tilde{f}_\cmf^{\leq d}, f_\cmf} \geq \inn{h^{\leq d}_\opt, f_\cmf}. \]

Our agnostic learner will look to approximate $\tilde{f}_\cmf^{\leq d}$ by outputting $p$ defined as follows. Suppose $f_\cmf = \sum_{I \in \N^n} \alpha_I H_I$, where $H_I$ is the multivariate Hermite polynomial of index $I$. For each $I$ of total degree at most $d$, which we denote as $|I| \leq d$, let $\beta_I$ be our estimate of $\alpha_I = \inn{f_\cmf, H_I}$ to within tolerance $\tau$ (to be determined). This can be done using $n^{O(d)}$ queries of tolerance $\tau$. Let $\tilde{f} = \sum_{|I| \leq d} \beta_I H_I$, and finally let $p = \frac{R}{\|\tilde{f}\|}\tilde{f}$. We have
\begin{align}
    \norm{\tilde{f}_\cmf^{\leq d} - p}^2 &= R^2 \norm*{ \frac{f_\cmf^{\leq d}}{\norm{f_\cmf^{\leq d}}} - \frac{\tilde{f}}{\norm{\tilde{f}}} }^2 \nonumber \\
    &= R^2 \norm*{ \frac{f_\cmf^{\leq d} - \tilde{f}}{\norm{f_\cmf^{\leq d}}} + \tilde{f}\left(\frac{1}{\norm{f_\cmf^{\leq d}}} - \frac{1}{\norm{\tilde{f}}}\right) }^2 \nonumber \\
    &\leq 2R^2 \left( \frac{\norm{f_\cmf^{\leq d} - \tilde{f}}^2}{\norm{f_\cmf^{\leq d}}^2} + \norm{\tilde{f}}^2\left(\frac{1}{\norm{f_\cmf^{\leq d}}} - \frac{1}{\norm{\tilde{f}}}\right)^2 \right) \nonumber \\
    &= 2R^2 \left( \frac{\norm{f_\cmf^{\leq d} - \tilde{f}}^2}{\norm{f_\cmf^{\leq d}}^2} + \left(\frac{\norm{f_\cmf^{\leq d}} - \norm{\tilde{f}}}{\norm{f_\cmf^{\leq d}}}\right)^2 \right) \nonumber \\
    &\leq 4R^2 \frac{\norm{f_\cmf^{\leq d} - \tilde{f}}^2}{\norm{f_\cmf^{\leq d}}^2} \tag*{\text{(triangle ineq.)}} \nonumber \\
    &\leq \frac{ 4R^2 n^d \tau^2}{\|f_\cmf^{\leq d}\|^2} \label{eq:upper-bound-calc},
\end{align}
since $\|\tilde{f} - f_\cmf^{\leq d}\| \le n^{d/2} \tau$.

We claim that we can assume WLOG that $\norm{\tilde{f}_\cmf^{\leq d}} \geq \epsilon/(2R)$. Indeed, we know $\max_{h \in \cH} \inn{h, f_\cmf} = \inn{h_\opt, f_\cmf}$ and also $\|h_\opt - h_\opt^{\leq d}\| \le \delta$. This implies that \[
R \|\tilde{f}_\cmf^{\leq d}\| = \inn{\tilde{f}_\cmf^{\leq d}, f_\cmf} \ge \inn{h_\opt^{\leq d}, f_\cmf} \ge \inn{h_\opt, f_\cmf} - C\delta,
\]
where the last inequality is Cauchy--Schwarz. If $\inn{h_\opt, f_\cmf} \le \eps$ then $0$ is a valid agnostic learner. Therefore, we can assume that $\inn{h_\opt, f_\cmf} \ge \eps$. Choosing $\delta = \frac{\eps}{2C}$, this means $\norm{\tilde{f}_\cmf^{\leq d}} \geq \epsilon/(2R)$.

By \cref{eq:upper-bound-calc}, we then have
\begin{equation}
    \|\tilde{f}_\cmf^{\leq d} - p\| \le \frac{ 4R n^{d/2} \tau}{\eps}. \label{eq:upper-bound-calc-2}
\end{equation}
    
Now observe that
\begin{align*}
   \inn{p, f_\cmf} &= \inn{\tilde{f}_\cmf^{\le d}, f_\cmf} + \inn{p - \tilde{f}_\cmf^{\le d}, f_\cmf}\\
   &\ge \inn{h_\opt^{\leq d}, f_\cmf} - \frac{ 4RC n^{d/2} \tau}{\eps} \tag*{\text{(\cref{eq:upper-bound-calc-2} and Cauchy--Schwarz)}} \\
   &= \inn{h_\opt, f_\cmf} + \inn{h_\opt^{\leq d} - h_\opt, f_\cmf} - \frac{ 4RC n^{d/2} \tau}{\eps} \\
   &\ge \inn{h_\opt, f_\cmf} - \frac{\eps}{2} - \frac{ 4RC n^{d/2} \tau}{\eps} \tag*{\text{(Cauchy--Schwarz, and using $\delta C$ = $\epsilon/2$)}}.
\end{align*}
Setting $\tau = \frac{\eps^2}{ 8RC n^{d/2}}$ gives us the desired result, namely that $\inn{p, f_\cmf} \geq \inn{h_\opt, f_\cmf} - \epsilon$. Thus we have the following theorem.

\begin{theorem}
    The class $\cH_{\relu}$ can be agnostically learned up to correlation $\epsilon$ (in the sense of \cref{as:base-learner}) using $n^{O(\epsilon^{-4/3})}$ queries of tolerance $n^{-\Theta(\epsilon^{-4/3})}\epsilon$. Similarly, $\cH_{\sigma}$ can be learned using $n^{\tilde{O}(\log^2 1/\epsilon)}$ queries of tolerance $n^{-\tilde{\Theta}(\log^2 1/\epsilon)}\epsilon^2$.
\end{theorem}
\begin{proof}
    Approximating the Hermite coefficients of degree at most $d$ takes $n^{O(d)}$ queries of tolerance $n^{-\Theta(d)}\epsilon$. As we show in \cref{app:approx-deg}, the $\delta$-approximate degree of unit-weight ReLUs is $O((1/\delta)^{4/3})$ and for unit-weight sigmoids it is $\tilde{O}(\log^2 1/\delta)$. The guarantees follow by the argument in the preceding discussion.
\end{proof}

We note that our lower bounds for ReLUs and sigmoids were for queries of tolerance $n^{-\Theta(\epsilon^{-1/12})}$ and $n^{-\Theta(\log^2 1/\epsilon)}$ respectively, which nearly matches these upper bounds.



\section*{Acknowledgements}
We thank the anonymous NeurIPS 2020 reviewers for their feedback.

\bibliography{refs}
\bibliographystyle{alpha}

\newpage
\appendix
\section{SQ lower bound subtleties}\label{app:sq-subtleties}
\subsection{Relationships between parameters}
When formally stating SQ lower bounds on learning $p$-concepts in terms of the statistical dimension, there are some subtleties to keep in mind. These have to do with the relationships between the query tolerance, the desired final error, and the norms of the functions in the class. Let us say our queries are of tolerance $\tau$, the final desired $L^2$ error $\norm{f - f^*}$ is $\epsilon$ (which corresponds to $L(f) - L(f^*) \leq \epsilon^2$; see \cref{eq:sqloss-diff}), and that the functions in $\cC$ satisfy $\|f^*\| \geq \beta$ for all $f^* \in \cC$. Then \begin{enumerate}
	\item We must have $\tau < \epsilon$. To see why, first note that for any query $\phi$ and two functions $f, g \in \cC$, a calculation shows that $|\Ex_{D_f}[\phi] - \Ex_{D_g}[\phi]| = |\inn{f - g, \tilde{\phi}}| \leq \norm{f - g}$, where $\tilde{\phi}(x) = (\phi(x, 1) - \phi(x, -1))/2$. Thus if one has a function $f$ such that $\epsilon < \norm{f - f^*} < \tau$, then no query of tolerance $\tau$ can tell them apart, but $f$ is not $\epsilon$-close to the target $f^*$.
	\item If $\epsilon \geq \beta$, a lower bound might not be possible. This is because the 0 function trivially achieves $L^2$ error $\norm{0 - f^*} = \norm{f^*}$. Imposing $\epsilon < \beta$ is sufficient to rule this out.
	\item We cannot arbitrarily rescale the p-concepts to increase $\beta$ since the functions must remain Boolean $p$-concepts. Rescaling would also increase the description length of the functions.
\end{enumerate}

The lower bound in \cref{thm:sq-lower-bound} (from \cite{goel2020superpolynomial}) is proved by reducing a distinguishing problem to a learning problem. For technical reasons, we end up requiring $\tau \leq \epsilon^2$,  $\epsilon \leq \beta/3$ for this reduction to go through. The points above show that these requirements are essentially necessary.

\subsection{The dependence of the query lower bound on the error \texorpdfstring{$\epsilon$}{ε} and the tolerance \texorpdfstring{$\tau$}{τ}}
The relationship between our query lower bounds, the desired error $\epsilon$, and the tolerance $\tau$ may seem a little unusual at first sight, especially the fact that the lower bounds seem to grow weaker as $\epsilon$ grows smaller. We make some clarifying remarks here.

Fundamentally, all SQ lower bounds are bounds on how many queries it takes to distinguish certain distributions from others. When discussing a concept class $\cC$, the distributions in question are the labeled distributions corresponding to concepts in the class. Learning $\cC$ is hard exactly insofar as it allows us to distinguish different labeled distributions arising from $\cC$. Many works in the SQ literature have this structure, but we will refer to \cite{goel2020superpolynomial} for formal statements.

Formally, the distinguishing problem we consider (\cite[Definition 4.2]{goel2020superpolynomial}) is that of distinguishing the labeled distribution $D_c$ arising from an unknown $c \in \cC$ from the reference distribution $D_0 = D \times \text{Unif}\{\pm 1\}$, using queries of tolerance at least $\tau$.

There are two crucial points to keep in mind here: \begin{enumerate}
    \item The distinguishing problem is a fundamentally information theoretic problem, and its difficulty scales only with $\tau$. In particular, using queries of tolerance $\tau$, we need at least $\sda(\cC, \tau^2)$ queries. This bound increases with $\tau$; in fact it often scales as $|\cC|\tau^2$ (see (\cite[Theorem 4.5 and Lemma 2.6]{goel2020superpolynomial}).
    \item The problem of learning up $\cC$ to error $\epsilon$ is hard exactly insofar as it allows us to solve the distinguishing problem (see \cite[Lemma 4.4]{goel2020superpolynomial}).
\end{enumerate}

An important consequence is that for fixed $\tau$, the query lower bound does not technically grow as a function of the error $\epsilon$: it applies uniformly for all $\epsilon$ small enough that it allows the learner to solve the distinguishing problem. In other words, there is a certain ``threshold'' $\epsilon_0$ such that for all $\epsilon \leq \epsilon_0$, the same query lower bound holds. As noted in point (3) of the previous subsection, this threshold can be taken to be $\beta/3$, where $\beta$ is such that $\norm{c} \geq \beta$ for all $c \in \cC$.

But at the same time, as noted in point (1) in the previous subsection, it is necessary that $\tau < \epsilon$ (and for the reduction it suffices to have $\tau \leq \epsilon^2$). If $\tau \geq \epsilon$, learning up to error $\epsilon$ is simply impossible.

With all this in mind, we can now answer the question of why our lower bounds seem to grow weaker as $\epsilon$ grows smaller: it is essentially because $\tau$ grows smaller as well, so that we get a series of incomparable (though still exponential) bounds due to the tradeoffs between query complexity, $\tau$, and $\epsilon$.

\section{Bounding the function norms of the \texorpdfstring{\cite{diakonikolas2020algorithms}}{[DKKZ20]} construction}\label{app:norm-calculation}
We shall consider the following slight rescaling of the functions of \cite{diakonikolas2020algorithms}. For activation functions $\psi, \phi : \R \to \R$, we have $g, f : \R^2 \to \R$ defined as follows.
\begin{gather*}
g(x) = \frac{1}{2m} \sum_{i = 1}^{2m} (-1)^i \phi\left(x_1 \cos \frac{i \pi}{m} + x_2 \sin \frac{i \pi}{m}\right) = \frac{1}{2m} \sum_{i = 1}^{2m} (-1)^i \phi\left(\dotp{x}{w_i}\right) \\
f(x) = \psi(g(x)),
\end{gather*} where $w_i = (\cos \frac{i \pi}{m}, \sin \frac{i \pi}{m})$. The number of hidden units is $k = 2m$. We will assume that $m$ is even.

The hard functions from $\R^n \to \R$ are then given by $f_A(x) = f(Ax)$ for certain matrices $A \in \R^{2 \times d}$ with $AA^T = I_2$. For $x \sim \cN(0, I_d)$, $Ax$ has the distribution $\cN(0, I_2)$. So for the purposes of the norm calculation, and hence throughout this section, we will work directly with $\cN(0, I_2)$. We will start by considering the norm of $g$. This can then be used to control the norm of $f$ via arguments similar to those in \cite{goel2020superpolynomial}.

\begin{lemma}\label{lem:norm-expression}
	Let $g : \R^2 \to \R$ be as defined above, and assume $m$ is even. Assume the standard Hermite expansion of $\phi$ is given by $\phi = \sum_a \hat{\phi}_a H_a$, where the $H_a$ are the normalized probabilists' Hermite polynomials. Under $\cN(0, I_2)$, \[ \|g\|^2 = \Omega \left(\sum_{\substack{a \gg m \\a \text{ even}}} \frac{\hat{\phi}_a^2}{\sqrt{a}}  \right). \] (For practical purposes, the asymptotic behavior of this expression is captured faithfully when we begin indexing from say $a = 100m$.)
\end{lemma}
\begin{proof}
	We have \begin{align*}
	\norm{g}^2 = \Ex[g(x)^2] &= \frac{1}{4m^2} \sum_{i, j = 1}^{2m} (-1)^i (-1)^j \Ex[\phi(\dotp{x}{w_i}) \phi(\dotp{x}{w_j})] \\
	&= \frac{1}{4m^2} \sum_{i, j = 1}^{2m} (-1)^i (-1)^j \Ex\left[\left( \sum_a \hat{\phi}_a H_a(\dotp{x}{w_i}) \right) \left( \sum_b \phi_b H_b(\dotp{x}{w_j}) \right)\right] \\
	&= \frac{1}{4m^2} \sum_{i, j = 1}^{2m} (-1)^i (-1)^j \left( \sum_a \hat{\phi}_a^2 \Ex[H_a(\dotp{x}{w_i})H_a(\dotp{x}{w_j})] \right).
	\end{align*} Now because $w_i, w_j$ are both unit vectors with $w_i \cdot w_j = \cos \frac{(i - j)\pi}{m}$, we have that $\dotp{x}{w_i}$ and $\dotp{x}{w_j}$ are both $\cN(0, 1)$ with covariance $ \cos \frac{(i - j)\pi}{m}$. Thus \begin{align*}
	\norm{g}^2 &= \frac{1}{4m^2} \sum_{i, j = 1}^{2m} (-1)^{i + j} \left( \sum_a \hat{\phi}_a^2 \cos^a \frac{(i - j)\pi}{m} \right) \\
	&= \frac{1}{4m^2} \sum_{i, j = 1}^{2m} (-1)^{i - j} \left( \sum_a \hat{\phi}_a^2 \cos^a \frac{(i - j)\pi}{m} \right),
	\end{align*} since $(-1)^{i+j}=(-1)^{i-j}$. Now, as we range over $i, j \in [2m]$, we see that $i-j = 0$ occurs $2m$ times, $i-j = 1$ occurs $2m-1$ times, and more generally $i-j = t$ occurs $2m - |t|$ times. Since a term with $i - j = t$ is exactly the same as one with $i - j = -t$ (by the evenness of $\cos$), we can say that for $t \neq 0$, $|i - j| = t$ occurs $2(2m - t)$ times. Thus the expression above can be written as \begin{align}
	\norm{g}^2 &= \frac{1}{4m^2} \left( 2m \left(\sum_a \hat{\phi}_a^2 \cos^a 0 \right) + \sum_{t=1}^{2m - 1} 2(2m - t)(-1)^t \left( \sum_a \hat{\phi}_a^2 \cos^a \frac{t \pi}{m} \right) \right) \nonumber \\
	&= \frac{1}{4m^2} \sum_{a} \hat{\phi}_a^2 \left( 2m + \sum_{t=1}^{2m-1} 2(2m - t) (-1)^t \cos^a \frac{t\pi}{m} \right) \nonumber \\
	&= \frac{1}{4m^2} \sum_{a} \hat{\phi}_a^2 S(a, m) \label{eq:norm-calculation} ,
	\end{align} where \[ S(a, m) = 2m + \sum_{t=1}^{2m-1} 2(2m - t) (-1)^t \cos^a \frac{t\pi}{m}. \] Now some algebraic manipulations are in order. By rewriting the index $t$ as $2m - t$, we get that \begin{align*} S(a, m) &= 2m + \sum_{t=1}^{2m-1} 2t (-1)^{2m - t} \cos^a \frac{(2m - t)\pi}{m} \\
	&= 2m + \sum_{t=1}^{2m-1} 2t (-1)^t \cos^a \frac{t\pi}{m}. \end{align*} Adding the two expressions for $S(a, m)$ and dividing by 2, we get \begin{align*}
	S(a, m) &= 2m + \sum_{t=1}^{2m - 1} 2m (-1)^t \cos^a \frac{t\pi}{m} \\
	&= 2m \sum_{t = 0}^{2m - 1} (-1)^t \cos^a \frac{t\pi}{m}
	\end{align*}
	 This sum vanishes when $a$ and $m$ have different parities, i.e.\ if $a$ is odd (recall that we assume $m$ is even). For even $a$, we have \[ S(a, m) = 4m \sum_{t = 0}^{m - 1} (-1)^t \cos^a \frac{t\pi}{m}. \] This is a trigonometric power sum with known closed form expressions. In particular, Equation 3.4 from \cite[\S 3]{da2017basic} (after correcting a typo) tells us that \begin{align*} T(a, m) = \sum_{t = 0}^{m - 1} (-1)^t \cos^a \frac{t\pi}{m} &= \begin{dcases}
	2^{1 - a} m \left( \sum_{p = 1}^{\floor{a/m}} \binom{a}{a/2 - pm/2} - \sum_{p=1}^{\floor{a/2m}} \binom{a}{a/2 - pm}  \right) & a \geq 2m \\
	2^{1 - a} m \sum_{p = 1}^{\floor{a/m}} \binom{a}{a/2 - pm/2} & m \leq a < 2m \\
	0 & a < m
	\end{dcases} \\
	&= \begin{dcases}
	2^{1 - a} m \left( \sum_{\substack{p = 1\\ p \text{ odd}}}^{\floor{a/m}} \binom{a}{a/2 - pm/2}  \right) & a \geq m \\
	0 & a < m
	\end{dcases}
	\end{align*}
	
	To get a sense for the asymptotics as $a \to \infty$, we consider $a \gg m$ (say $a \geq 100m$). In this regime the sum of binomial coefficients in the sum above is seen to be $\Omega(2^a/\sqrt{a})$ (the $p=1$ term alone contributes roughly $\binom{a}{a/2}$), and we get that $T(a, m) = \Omega(m/\sqrt{a})$.
	
	This means $S(a, m) = 0$ for odd $a$ and $S(a, m) = 4m T(a, m) = \Omega(m^2/\sqrt{a})$ for large, even $a$. Substituting this back into \cref{eq:norm-calculation}, we get that \[ \|g\|^2 = \Omega \left(\sum_{\substack{a \gg m \\a \text{ even}}} \frac{\hat{\phi}_a^2}{\sqrt{a}}  \right). \]
\end{proof}

We can now consider the special cases of $\phi = \relu$ and $\phi = \sigma$ (the standard sigmoid) that are of interest.

\begin{corollary}\label{cor:g-norm-relu}
	Consider $g$ instantiated with $\phi = \relu$. Then $\|g\| = \Omega(1/m)$.
\end{corollary}
\begin{proof}
	The Hermite coefficients of $\relu$ satisfy $\hat{\phi}_a = \Theta(a^{-5/4})$ (\cref{lem:relu-hermite-exp}). Thus by \cref{lem:norm-expression}, \[ \|g\|^2 = \Omega( \sum_{\substack{a \geq 100m \\a \text{ even}}} a^{-3} ) = \Omega(1/m^2). \]
\end{proof}

\begin{corollary}\label{cor:g-norm-sigmoid}
	Consider $g$ instantiated with $\phi = \sigma$, the standard sigmoid. Then $\|g\| = e^{-O(\sqrt{m})}$.
\end{corollary}
\begin{proof}
	The Hermite coefficients of $\sigma$ asymptotically satisfy $\hat{\phi}_a \simeq e^{-C\sqrt{a}}$ \cite[\S A.2]{goel2020superpolynomial} for some $C$. Thus by \cref{lem:norm-expression}, \[ \|g\|^2 = \Omega( \sum_{\substack{a \geq 100m \\a \text{ even}}} \frac{e^{-\sqrt{a}}}{\sqrt{a}} ). \] The result then follows by the following standard integral approximation: \[ \sum_{t = N}^\infty \frac{e^{-\sqrt{t}}}{\sqrt{t}} \approx \int_{N}^\infty \frac{e^{-\sqrt{t}}}{\sqrt{t}} \ dt = 2e^{-\sqrt{N}}. \]
\end{proof}

We can now translate these into norm lower bounds on $f = \psi \circ g$. For us it suffices to consider $\psi = \tanh : \R \to [-1, 1]$, which is essentially the sigmoid centered at 0. The centering at 0 and the output range being $[-1, 1]$ is what is important to us, because we use $f$ to capture the conditional mean function of a $p$-concept.

\begin{lemma}\label{lem:f-norm-lower-bound-relu}
	Consider $f$ instantiated with $\psi = \tanh$ and $\phi = \relu$. Then $\|f\| = \Omega(1/m^6)$.
\end{lemma}
\begin{proof}
	Ideally we would like to use the norm bound on $g$ to obtain an anti-concentration inequality of the form $\Pr[|g(x)| > t]$, and then translate that into a norm lower bound for $f$, but this is not immediate because $g$ is unbounded. So we introduce the function $g^T$, which is the same as $g$ except with the truncated ReLU, $\relu^T(x) = \min(T, \relu(x))$ ($T$ to be determined), in place of all standard ReLUs. Clearly $|g^T(x)| \leq T$ for all $x$. It is also easy to see by a union bound that \[ \Pr[g(x) \neq g^T(x)] \leq 2m \Pr_{t \sim \cN(0,1)}[\relu(t) \neq \relu^T(t)] \leq 2m e^{-T^2/2}, \]
	since each $w_i$ is a unit vector.
	
	Let $\relu_{w}(x)$ be shorthand for $\relu(\dotp{x}{w})$, and similarly $\relu^T_{w}$. Observe first that \begin{align*}
	    \norm{g - g^T} &= \frac{1}{2m}\norm*{\sum_{i=1}^{2m} (-1)^i ( \relu_{w_i} - \relu^T_{w_i})} \\
	    &\leq \frac{1}{2m} \sum_{i=1}^{2m} \norm{\relu_{w_i} - \relu^T_{w_i})} \\
	    &= \norm{\relu - \relu^T}_{\cN(0, 1)} \\
	    &\leq \sqrt{e^{-\frac{T^2}{2}} \left( T^2 + 1 - \frac{T}{\sqrt{2\pi}} \right)}
	\end{align*} where the third equality again uses the fact the $w_i$ are unit vectors, and the last inequality is \cref{lem:relu-truncation}. By picking $T = \Theta(m)$, this coupled with the fact that $\norm{g} = \Omega(1/m)$ (\cref{cor:g-norm-relu}) tells us that $\norm{g^T} = \Omega(1/m)$ as well.
	
	This bound on $\norm{g^T}$ yields an anti-concentration inequality for $g^T$ as follows: \[ \|g^T\|^2 = \Ex[g^T(x)^2] \leq t^2 \Pr[|g^T(x)| \leq t] + T^2  \Pr[|g^T(x)| > t] = t^2 + (T^2 - t^2)\Pr[|g^T(x)| > t], \] so that \[ \Pr[|g^T(x)| > t] \geq \frac{\|g^T\|^2 - t^2}{T^2 - t^2}. \] Recall that $\Pr[g(x) \neq g^T(x)] \leq 2m e^{-T^2/2}$, so \[ \Pr[|g(x)| > t] \geq \frac{\|g^T\|^2 - t^2}{T^2 - t^2} - 2m e^{-T^2/2}. \] Thus by taking $T = \Theta(m)$ and $t = \Theta(1/m)$, we get that \[\Pr[|g(x)| > \Theta(1/m)] \geq \Omega(1/m^4). \] Thus finally we have \[ \|f\| = \Ex[\tanh(g(x))^2] \geq \tanh^2(\Theta(1/m)) \Omega(1/m^4) \geq \Omega(1/m^6), \] since $\tanh(x) \approx x - x^3$ for small $x$ (by its Taylor series).
\end{proof}

\begin{lemma}[\cite{goel2020superpolynomial}, Appendix A.1]\label{lem:relu-truncation}
For $\relu^T(x) = \min(T, \relu(x))$,
\[ \norm{\relu - \relu^T}_{\cN(0, 1)} \leq \sqrt{e^{-\frac{T^2}{2}} \left( T^2 + 1 - \frac{T}{\sqrt{2\pi}} \right)}. \]
\end{lemma}
\begin{proof}
    Let $p(t) = \frac{1}{\sqrt{2\pi}}e^{-t^2/2}$ be the pdf of $\cN(0, 1)$. Then \begin{align*}
	    \norm{\relu - \relu^T}^2_{\cN(0, 1)} &= \Ex_{t \sim \cN(0,1)}\left[\left(\relu(t) - \relu^T(t)\right)^2\right] \\ &= \int_{T}^{\infty} (t - T)^2 p(t)\ dt \\ &= \int_{T}^{\infty} t^2 p(t)\ dt - 2T\int_{T}^{\infty} t p(t)\ dt + T^2 \int_{T}^{\infty} p(t)\ dt 	\end{align*} Noting that $p'(t) = -tp(t)$, we have
	\begin{align*}
	& \int_{T}^\infty t^2 p(t)\ dt  = \int_{T}^\infty -t\, d(p(t)) \\
	&\qquad = -t\,p(x)\bigg|_{T}^\infty + \int_{T}^\infty p(t)\ dt \tag*{\text{(integration by parts)}} \\
	&\qquad = T\, p(T) + \Pr_{t \sim \mathcal{N}(0, 1)}(t > T), \\
	& \int_{T}^\infty t\, p(t)\ dt = -p(t)\bigg|_{T}^\infty = p(T), \\
	& \int_{T}^\infty p(t)\ dt = \Pr_{t \sim \mathcal{N}(0, 1)}(t > T) \le e^{-\frac{T^2}{2}}.
	\end{align*} The claim follows by algebra.
\end{proof}

\begin{lemma}\label{lem:f-norm-lower-bound-sigmoid}
	Consider $f$ instantiated with $\psi = \tanh$ and $\phi = \sigma$. Then $\|f\| = e^{-O(\sqrt{m})}$.
\end{lemma}
\begin{proof}
	Here the same approach as above becomes considerably simpler since $|g(x)| \leq 1$ always. The norm bound on $g$ yields the following anti-concentration inequality: \[ \Pr[|g(x)| > t] \geq \frac{\|g\|^2 - t^2}{1 - t^2}. \] In our case, taking $t = e^{-C\sqrt{m}}$ for sufficiently large $C$ and using $\|g\| = e^{-O(\sqrt{m})}$ (\cref{cor:g-norm-sigmoid}) yields \[ \Pr[|g(x)| > e^{-C\sqrt{m}}] = e^{-O(\sqrt{m})}. \] Thus \[ \|f\| = \Ex[\tanh(g(x))^2] \geq \tanh^2(e^{-C\sqrt{m}}) e^{-O(\sqrt{m})} \geq e^{-O(\sqrt{m})}, \] since again $\tanh(x) \approx x - x^3$ for small $x$.
\end{proof}

\section{Approximate degree of ReLUs and sigmoids}\label{app:approx-deg}
Here we give estimates for the $\delta$-approximate degree of ReLUs and sigmoids under the standard Gaussian using bounds on their Hermite coefficients. Recall that we consider units $\phi(\dotp{w}{x})$ with $\|w\|_2 \leq 1$. It is clear that for $\phi = \relu$ and $\phi = \sigma$, the norm only increases monotonically with $\|w\|_2$, so for the purposes of analysis it suffices to consider exactly $\|w\|_2 = 1$.

It is not hard to show that whenever $w$ is a unit vector, the total-degree-$d$ Hermite weight of $\phi(\dotp{w}{x})$ as $x \sim \cN(0, I_n)$ is the same as that of the univariate $\phi(t)$ as $t \sim \cN(0,1)$. (A quick way of seeing this is to note that by rotational symmetry, we may assume WLOG that $w = e_1$, in which case the calculation is very straightforward.)

In what follows, we say $\hat{\phi}_a$ are the Hermite coefficients of $\phi : \R \to \R$ if $\phi = \sum_a \hat{\phi}_a H_a$, where the $H_a$ are the normalized probabilists' Hermite polynomials. We use $\tilde{H}_a$ to denote the un-normalized (i.e.\ monic) Hermite polynomials. (Note that this is somewhat nonstandard notation.)

First we consider ReLUs.
\begin{lemma}\label{lem:relu-hermite-exp}
$\hat{\relu}_0 = 1/\sqrt{2\pi}$, $\hat{\relu}_1 = 1/2$ and for $a \ge 2$, $\hat{\relu}_a = \frac{1}{\sqrt{2\pi a!}}(\tilde{H}_{a}(0) + a\tilde{H}_{a-2}(0))$. In particular, $\hat{\relu}_a = 0$ for odd $a \geq 3$ and $|\hat{\relu}_a| = \Theta(a^{-5/4})$ for even $a$.
\end{lemma}
\begin{proof}
We use the following standard recurrence relation: $\tilde{H}_{a+1}(x) = x\tilde{H}_{a}(x) - a\tilde{H}_{a-1}(x)$. For $a \ge 2$,
\begin{align*}
    \hat{\relu}_a &= \frac{1}{\sqrt{2 \pi}}\int_{-\infty}^{\infty} \relu(x) H_a(x) e^{-\frac{x^2}{2}} dx\\
    &= \frac{1}{\sqrt{2 \pi a!}}\int_{0}^{\infty} x \tilde{H}_a(x) e^{-\frac{x^2}{2}} dx \\
    &= \frac{1}{\sqrt{2 \pi a!}}\int_{0}^{\infty} (\tilde{H}_{a + 1}(x) + a\tilde{H}_{a-1}(x)) e^{-\frac{x^2}{2}} dx \\
    &= \frac{1}{\sqrt{2 \pi a!}} (\tilde{H}_{a}(0) + a\tilde{H}_{a-2}(0)).
\end{align*} Since $\tilde{H}_a(0) = 0$ for odd $a$, $\hat{\relu}_a = 0$ as well. For even $a = 2b$ with $b \ge 2$, by standard expressions for $\tilde{H}_a(0)$, we have
\begin{align*}
\hat{\relu}_a &= \frac{1}{\sqrt{2\pi (2b)!}}(\tilde{H}_{2b}(0) + 2b \tilde{H}_{2b - 2}(0))\\
&= \frac{1}{\sqrt{2\pi (2b)!}} \left((-1)^b\frac{(2b)!}{b!2^b} + 2b (-1)^{b-1}\frac{(2b - 2)!}{(b-1)!2^{b-1}}\right)\\
&= \frac{(-1)^b\sqrt{(2b)!}}{\sqrt{2\pi}b! 2^b} \left(1 - \frac{2b}{2b-1}\right)\\
&= \frac{(-1)^{b+1}\sqrt{(2b)!}}{\sqrt{2\pi}(2b - 1)b! 2^b} \\
&\eqsim \frac{(-1)^{b+1}}{\sqrt{2\pi}(2b - 1)(2b)^{1/4}} \\
&\eqsim \frac{(-1)^{b+1}}{b^{5/4}}\\
\end{align*}
Here the second inequality follows from the fact $\binom{n}{n/2} \eqsim \frac{2^{n/2}}{\sqrt{n}}$.
\end{proof}

\begin{corollary}
    The $\delta$-approximate degree of $\relu$ under $\cN(0,1)$ is $O(\delta^{-4/3})$.
\end{corollary}
\begin{proof}
    Let $p$ denote the the Hermite expansion of $\relu$ truncated at degree $d$. By the fact that $|\hat{\relu}_a| = \Theta(a^{-5/4})$ for even $a$ (and 0 for odd $a$), we see that \begin{align*}
        \norm{p - \relu}^2 &= \sum_{a > d} \hat{\relu}_a^2 \\
        &= \sum_{\substack{a > d\\a \text{ even}}} \Theta(a^{-5/2}) \\
        &= \Theta(d^{-3/2}).
    \end{align*} For this to be at most $\delta^2$, we only need $d = O(\delta^{-4/3})$.
\end{proof}

Now we turn to sigmoids. Let $\sigma$ denote the standard sigmoid, i.e.\ the logistic function $\sigma(t) = 1/(1 + e^{-t})$.
\begin{lemma}
For all sufficiently large $a$, $\hat{\sigma}_a = e^{-\Omega(\sqrt{a})}$.
\end{lemma}
\begin{proof}
    Upper bounds on the Hermite coefficients of sigmoidal funtions are known to follow from classic results in the complex analysis of Hermite series \cite{hille1940contributions, boyd1984asymptotic}. We refer to \cite[Corollary F.7.1]{panigrahi2019effect}, where this computation is done for $\tanh'(x) = 1 - \tanh^2(x)$. The calculation is very similar for $\sigma$ (in fact, $\sigma$ is just an affine shift of $\tanh$).
\end{proof}

\begin{corollary}
    The $\delta$-approximate degree of $\sigma$ under $\cN(0,1)$ is $\tilde{O}(\log^2 1/\delta)$.
\end{corollary}
\begin{proof}
    Let $p$ denote the Hermite expansion of $\sigma$ truncated at degree $d$. Observe that \begin{align*}
        \norm{\sigma - p}^2 &= \sum_{a > d} \hat{\sigma}_a^2 \\
        &= \sum_{a > d} e^{-\Omega(\sqrt{a})} \\
        &= \Theta(\sqrt{d}e^{-\Omega(\sqrt{d})}),
    \end{align*} which is at most $\delta^2$ for $d = \tilde{O}(\log^2 1/\delta)$.
\end{proof}

\section{Frank--Wolfe convergence guarantee}\label{app:fw-proof}
Here we provide a self-contained proof of \cref{thm:fw-guarantee}, restated here. In fact, we generalize the analysis to handle any constant factor approximation to the optimum, meaning that in the Frank--Wolfe subproblem of \cref{alg:fw-generic}, we only require \begin{equation}
 \inn{s, -\nabla \cfunc(z_t) } \geq \alpha \max_{s' \in \cZ'}\inn{s', -\nabla \cfunc(z_t)} - \frac{1}{2} \delta \gamma_t C_\cfunc \label{eq:constant-factor-opt} \end{equation} for some constant $\alpha \leq 1$. We closely follow \cite[Appendix A]{jaggi2013revisiting}, noting the differences in our slightly more general setup (the standard setup has $\cZ' = \cZ$, and $\alpha = 1$).
\begin{theorem}
	Let $\cZ' \subseteq \cZ$ be convex sets, and let $\cfunc : \cZ \to \R$ be a $\beta$-smoothly convex function. Let $C_p = \beta \diam(\cZ)^2$. Suppose that $z^* \in \cZ'$ achieves $\min_{z' \in \cZ'}p(z')$. For every $t$, the iterates of \cref{alg:fw-generic} (modified to work with \cref{eq:constant-factor-opt}) satisfy \[ \cfunc(z_t) -  \cfunc(z^*) \leq \frac{2C_\cfunc}{\alpha^2(t + 2)}(1 + \delta). \]
\end{theorem}
\begin{proof}
	Define the duality gap function $q : \cZ \to \R$ as \[ q(z) = \max_{s \in \cZ'} \inn{z - s, \nabla p(z)}. \] Notice that $q$ takes in any $z \in \cZ$ but maximizes only over $s \in \cZ'$. By convexity of $p$ over $\cZ$, we know that for all $z \in \cZ, s \in \cZ'$, $p(z) + \inn{s - z, \nabla p(z)} \leq p(s)$, meaning that $p(z) - p(s) \leq q(z)$. In particular, $p(z) - p(z^*) \leq q(z)$, so that $q(z)$ always provides an upper bound on the gap between $p(z)$ and $p(z^*)$ --- this is weak duality.
	
	Next we establish the following guarantee on the progress made in each step, which corresponds to Lemma 5 in Jaggi's proof.
	\begin{claim} Let the $t\th$ step be $z_{t+1} = z_t + \gamma(s - z_t)$, where $z_t, z_{t+1}, s \in \cZ$, $\gamma \in[0, 1]$ is arbitrary, and $s$ satisfies \[ \inn{s, -\nabla p(z_t)} \geq \alpha \max_{s' \in \cZ'} \inn{s', -\nabla p(z_t)} - \frac{1}{2}\delta \gamma C_p. \]  Then we have \[ p(z_{t+1}) \leq p(z_t) - \alpha \gamma q(z_t) + \frac{\gamma^2}{2}C_p(1 + \delta). \]
	\end{claim}
	To see this, first note that because $p$ is $\beta$-smoothly convex, \begin{align*}
	p(z_{t+1}) &= p(z_t + \gamma(s - z_t)) \\
	&\leq p(z_t) + \gamma \inn{s - z_t, \nabla p(z_t)} + \frac{\gamma^2}{2}C_p. 
	\end{align*} And from the way $s \in \cZ$ was picked, we have \begin{align*}
	\inn{s - z_t, -\nabla p(z_t)} &\geq \alpha \max_{s' \in \cZ'}\inn{s' - z_t, -\nabla p(z_t)} - \frac{1}{2}\delta \gamma C_p \\
	&= \alpha q(z_t) - \frac{1}{2}\delta \gamma C_p.
	\end{align*} The claim now follows.
	
	As a consequence of the claim, we can say \begin{align*}
	p(z_{t+1}) - p(z^*) &\leq p(z_t) - p(z^*) - \gamma q(z_t) +  \frac{\gamma^2}{2}C_p(1 + \delta) \\
	&\leq (1 - \alpha \gamma) (p(z_t) - p(z^*)) +  \frac{\gamma^2}{2}C_p(1 + \delta),
	\end{align*} since $q(z_t) \geq p(z_t) - p(z^*)$ (weak duality). Taking $\gamma = \gamma_t = \frac{2}{\alpha(t+2)}$, the following bound can now by proven by induction on $t$: \[ p(z_t) - p(z^*) \leq \frac{2}{\alpha^2(t+2)} C_p(1 + \delta). \] This proves the theorem. 
\end{proof}

\section{Relationship between Boolean 0-1 loss and real-valued correlation loss}\label{app:boolean-loss}
Let $\cD$ be a distribution on $\R^n \times \R$. Our lower bound applies against agnostic learners that satisfy \cref{as:base-learner}, with a real-valued correlation guarantee, i.e.\ learners that learn a class $\cH$ by outputting $f : \R^n \to \R$ such that \begin{equation}
    \Ex_{(x, y) \sim \cD}[f(x) y] \geq \max_{g \in \cH} \Ex_{(x, y) \sim \cD}[g(x) y] - \epsilon.
\end{equation}

In the Boolean setting, where the labels are $\{\pm 1\}$-valued, we have a distribution $P$ on $\R^n \times \{\pm 1\}$. A learner is said to agnostically learn $\cH$ in terms of 0-1 loss if it is able to output $f : \R^n \to \{\pm 1\}$ such that \begin{align*}
    \Pr_{(a, b) \sim P}[f(a) \neq b] \leq \min_{g \in \cH} \Pr_{(a, b) \sim P}[g(a) \neq b] + \epsilon,
\end{align*} or equivalently \begin{align*}
    \Ex_{(a, b) \sim P}[f(a) b] \geq \max_{g \in \cH} \Pr_{(a, b) \sim P}[g(a)b] - \epsilon/2,
\end{align*} since $\Ex_{(a, b) \sim P}[f(a) b] = 1 - 2\Pr_{(a, b) \sim P}[f(a) \neq b]$. (The latter formulation has the benefit of making sense even for real-valued $f : \R^n \to \R$.)

It is not obvious that a learner $L$ of the above kind (with a Boolean 0-1 loss guarantee) gives us a real-valued correlation loss guarantee, because it only knows how to operate on distributions $P$ on $\R^n \times \{\pm 1\}$ (with Boolean labels), not distributions $\cD$ on $\R^n \times \R$ (with arbitrary real labels). Moreover, in the SQ setting, we must be able to translate $L$'s queries to $P$, which are of the form $\phi : \R^n \times \{\pm 1\} \to \R$, into queries to $\cD$. We claim that both of these difficulties can be gotten around. We will show that if $\cD$ has bounded labels, say in $[-C, C]$, we can construct a distribution $P$ on $\R^n \times \{\pm 1\}$ and simulate $L$ on $P$ to obtain a correlation loss guarantee wrt $\cD$.

Indeed, let $D$ denote the marginal of $\cD$ on $\R^n$; for us, $D$ is always $\cN(0, I_n)$. Then $P$ can be constructed simply as follows: draw $a \sim D$, and then randomly pick $b \in \{\pm 1\}$ such that $\Ex[b|a] = (\Ex_{(x,y) \sim \cD}[y|x=a])/C$. (One could think of this as the ``$p$-concept trick''.) Equivalently, pick \[ b = \begin{cases} 1 &\text{ with probability } \frac{1 + (\Ex_{(x,y) \sim \cD}[y|x=a])/C}{2} \\ -1 &\text{ otherwise} \end{cases} \] One can easily see that for any $f : \R^n \to \R$, \[ \Ex_{(a, b) \sim P}[f(a) b] = \frac{1}{C} \Ex_{(x, y) \sim \cD}[f(x) y], \] so that using $L$ to learn up to 0-1 error $\epsilon$ gives a correlation loss guarantee up to $C \epsilon/2$. It remains to show that we can indeed simulate $L$'s queries to $P$ using only SQ access to $\cD$. For any query $\phi : \R^n \times \{\pm 1\} \to \R$, observe that (since the marginal of $P$ on $\R^n$ is also $D$) \begin{align*}
    \Ex_{(a, b) \sim P} [\phi(a, b)] &= \Ex_{a \sim D} \left[ \phi(a, 1) \frac{1 + (\Ex_{(x,y) \sim \cD}[y|x=a])/C}{2} + \phi(a, -1) \frac{1 - (\Ex_{(x,y) \sim \cD}[y|x=a])/C}{2} \right] \\
    &= \frac{1}{2}\Ex_{a \sim D}[\phi(a, 1) + \phi(a, -1)] + \frac{1}{2C}\Ex_{(x, y) \sim \cD} [(\phi(x, 1) - \phi(x, -1) y].
\end{align*} This expression can be computed using two statistical queries to $\cD$ (or even just one, since we know the marginal $D$).

In our reduction (\cref{thm:boosting-surloss}), we end up using the base learner on labeled distributions $\cD$ where the labels correspond to the loss functional's gradient; when using surrogate loss, the label for $x$ is $\psi(f^*(x)) - \psi(f(x))$. We see that this is indeed bounded in $[-2, 2]$, since $\psi : \R \to [-1,1]$. Recall that in solving the Frank--Wolfe subproblem we needed to worry about simulating SQ access to this $\cD$ using only SQ access to the true $D_{\psi \circ f^*}$ (see \cref{eq:rewriting-fw} and surrounding discussion). Here we actually have a further layer: we need to simulate SQ access to $P$ using SQ access to $\cD$, itself simulated using actual SQ access to $D_{\psi \circ f^*}$. But it is easily verified that by the argument just outlined, no trouble arises here, and that one can in fact also ``directly'' simulate $P$ using $D_{\psi \circ f^*}$ by the same argument as used for \cref{eq:rewriting-fw}.

\section{Relationship between square loss and correlation loss for ReLUs}\label{app:corr-loss}

Let $\cD$ be a distribution on $\R^n \times \R$, and assume the labels are bounded in $[-C, C]$. Our lower bounds apply to agnostic learners that satisfy \cref{as:base-learner}, with a guarantee in terms of correlation, where the output hypothesis $f$ must satisfy \[ \Ex_{(x, y) \sim \cD}[f(x) y] \geq \max_{g \in \cH} \Ex_{(x, y) \sim \cD}[g(x) y] - \epsilon. \] But agnostic learning of real-valued functions is usually phrased in terms of square loss: \[ \Ex_{(x, y) \sim \cD}[(f(x) - y)^2] \leq \min_{g \in \cH} \Ex_{(x, y) \sim \cD}[(g(x) - y)^2] + \epsilon'. \] Here we show that for the class of ReLUs, $\cH = \cH_{\relu}$, an agnostic learner $L$ with a square loss guarantee can be used to satisfy \cref{as:base-learner}. Fundamentally, this amounts to working out a geometric relationship between distances and projections in our function space, and much of the following argument can be viewed as a somewhat careful elaboration of what, in the familiar Euclidean setup, is more easily visualized.

For simplicity, throughout this section we will scale the class $\cH_{\relu}$ so that the maximum norm of any function is 1: \[ \cH = \cH_{\relu} = \{ \pm \sqrt{2} \relu(\dotp{u}{x}) \mid \|u\|_2 \leq 1\}. \] An important property of this class is that we can always scale a function $h \in \cH$ to have any desired norm in $[0, 1]$ without leaving the class. That is, for any nonzero $h \in \cH$ and any $\lambda \in [0, 1]$, $\frac{\lambda}{\norm{h}}h \in \cH$. This follows simply from the fact that $\norm{\relu(\dotp{u}{x})} = \|u\|_2 / \sqrt{2}$. We can think of this as saying that $\cH$ is a norm-bounded section of a convex cone.

Let $f_\cmf(x) = \Ex[y|x]$. Let $h_{\sq}$ be a minimizer over all $h \in \cH$ of the squared loss, $\Ex_{(x, y) \sim \cD}[(h(x) - y)^2]$. An equivalent and more convenient view is that this is a minimizer of the squared distance $\norm{h - f_\cmf}^2$, since \[ \norm{h - f_\cmf}^2 = \norm{h}^2 - 2\inn{h, f_\cmf} + \norm{f_\cmf}^2 = \Ex_{\cD}[(h(x) - y)^2] + \norm{f_\cmf}^2 - \Ex_{\cD}[y^2], \] and the latter terms are independent of $h$. This view is particularly important since it, combined with the fact that $\cH$ is essentially a bounded convex cone, gives us an orthogonal projection theorem. Specifically, it is the case that the norm of $h_\sq$ must be the length of the projection of $f_\cmf$ onto the line $\lambda h_\sq$ for $\lambda \in [0, 1]$ (assuming this length is at most 1; otherwise, the norm is 1). In other words, \begin{equation} \norm{h_\sq} = \min \{ \inn{\frac{h_\sq}{\norm{h_\sq}}, f_\cmf}, 1 \}. \label{eq:projection-thm} \end{equation} This can be seen by asking: for what $\lambda \in [0, 1]$ is $\norm{\frac{\lambda}{\norm{h_\sq}}h_\sq - f_\cmf}$ minimized? (The point being that $h_\sq$ could be rescaled to have norm $\lambda$.) By writing this as \[ \norm{\frac{\lambda}{\norm{h_\sq}}h_\sq - f_\cmf}^2 = \left(\lambda - \inn{\frac{h_\sq}{\norm{h_\sq}}, f_\cmf}\right)^2 + \norm{f_\cmf}^2 - \inn{\frac{h_\sq}{\norm{h_\sq}}, f_\cmf}^2, \] the observation follows immediately.\footnote{Note that here we are assuming $\inn{h_\sq, f_\cmf} \geq 0$ WLOG, since otherwise we would consider $-h_\sq$.} This projection theorem also tells us that $h_\sq = 0$ iff $f_\cmf$ has no projection onto any $h \in \cH$, i.e.\ $\inn{h, f_\cmf} = 0$ for all $h \in \cH$.\footnote{For another way to see this, for any nonzero $h \in \cH$, expand $\norm{\lambda h - f_\cmf}^2 \geq \norm{0 - f_\cmf}^2$ and let $\lambda \to 0$.}

Let $h_\cor$ be a maximizer of the correlation, $\Ex_{(x, y) \sim \cD}[h(x)y] = \inn{h, f_\cmf}$. We may clearly assume that $h_\cor$ has the maximum possible norm, which is 1. We claim that in fact, $h_\cor$ can be taken to be $h_\sq / \norm{h_\sq}$ (assuming $h_\sq \neq 0$; otherwise, $h_\cor = 0$ as well since, as noted, this means $\inn{h, f_\cmf} = 0$ for all $h \in \cH$). To see why, first assume $h_\sq \neq 0$ and use the fact that for any nonzero $h \in \cH$, the square loss achieved by $\frac{\norm{h_\sq}}{\norm{h}}h$ (i.e.\ $h$ scaled to have $h_\sq$'s norm) cannot be better than that of $h_\sq$ itself. Thus by an algebraic manipulation we have \begin{align*}
\norm{h_\sq - f_\cmf}^2 &\leq \norm*{\frac{\norm{h_\sq}}{\norm{h}}h - f_\cmf}^2 \\
\implies \inn{\frac{h_\sq}{\norm{h_\sq}}, f_\cmf} &\geq \inn{\frac{h}{\norm{h}}, f_\cmf} \geq \inn{h, f_\cmf}.
\end{align*} Since this holds for any $h \in \cH$, we may take $h_\cor = h_\sq / \norm{h_\sq}$.

Now suppose we have an agnostic learner in terms of square loss that returns $h$ such that \[ \norm{h - f_\cmf}^2 \leq \norm{h_\sq - f_\cmf}^2 + \epsilon'. \] For a suitable choice of $\epsilon'$ (depending on the final desired $\epsilon$), we would like to say that $h/\norm{h}$ achieves correlation that is $\epsilon$-competitive with $h_\cor$. Indeed, if $h_\sq = 0$ this is trivial, since as noted this means $\inn{h, f_\cmf} = 0$ for all $h \in \cH$. Otherwise, by comparing $\frac{\norm{h}}{\norm{h_\sq}}h_\sq$ (i.e.\ $h_\sq$ scaled to have $h$'s norm) with $h_\sq$ itself, we may say that \[ \norm{h - f_\cmf}^2 \leq \norm{h_\sq - f_\cmf}^2 + \epsilon' \leq \norm*{\frac{\norm{h}}{\norm{h_\sq}}h_\sq - f_\cmf}^2 + \epsilon'. \] Some rearrangement gives \begin{align} \inn{\frac{h}{\norm{h}}, f_\cmf} &\geq \inn{\frac{h_\sq}{\norm{h_\sq}}, f_\cmf} - \frac{\epsilon'}{2\norm{h}} \nonumber \\ &= \inn{h_\cor, f_\cmf} - \frac{\epsilon'}{2\norm{h}}, \label{eq:h-vs-hcor} \end{align} showing that $h/\norm{h}$ is $\frac{\epsilon'}{2\norm{h}}$-competitive with $h_\cor$.

But an issue here is that $\norm{h}$ could be very small, or even zero. We claim that we can actually address this separately as an easy case: it implies that we are in a trivial situation in which even the 0 function performs fairly well, and so even the best possible correlation must be quite small.

\begin{lemma}\label{lem:h-small-norm}
	Let $h$ be such that $\norm{h - f_\cmf}^2 \leq \norm{h_\sq - f_\cmf}^2 + \epsilon'$. Suppose $\norm{h} \leq \eta$. Then $\inn{h_\cor, f_\cmf} \leq \sqrt{\epsilon' + 2C\eta}$. In particular, the $0$ function is $\sqrt{\epsilon' + 2C\eta}$-competitive with $h_\cor$.
\end{lemma}
\begin{proof}
By Cauchy--Schwarz, \[ \norm{0 - f_\cmf}^2 - \norm{h - f_\cmf}^2 = 2\inn{h, f_\cmf} - \norm{f_\cmf}^2 \leq 2\norm{h}\norm{f_\cmf} \leq 2C\eta, \] where we use $\norm{f_\cmf} \leq C$ since the labels are assumed to be bounded in $[-C, C]$. Thus \[ \norm{0 - f_\cmf}^2 \leq \norm{h - f_\cmf}^2 + 2C\eta \leq \norm{h_\sq - f_\cmf}^2 + \epsilon' + 2C\eta. \] On the other hand, by definition of $h_\sq$, \[ \norm{h_\sq - f_\cmf}^2 \leq \norm{0 - f_\cmf}^2, \] Put together, this means that the 0 function achieves nearly the same square loss as $h_\sq$: \begin{equation}
\norm{h_\sq - f_\cmf}^2 \leq \norm{0 - f_\cmf}^2 \leq \norm{h_\sq - f_\cmf}^2 + \epsilon' + 2C\eta \label{eq:hsq-ineq}.
\end{equation} This lets us conclude that $\norm{h_\sq}$ must be small: \[ \norm{h_\sq}^2 = \norm{f_\cmf}^2 - \norm{h_\sq - f_\cmf}^2 + 2\inn{h_\sq - f_\cmf, h_\sq} \leq \epsilon' + 2C\eta, \] where we use \cref{eq:hsq-ineq} and the fact that by can rewrite \cref{eq:projection-thm} as $\norm{h_\sq} \leq \inn{\frac{h_\sq}{\norm{h_\sq}}, f_\cmf}$, or $\inn{h_\sq - f_\cmf, h_\sq} \leq 0$. But now since $\norm{h_\sq} \leq \sqrt{\epsilon' + 2C\eta} < 1$ ($\epsilon'$ and $\eta$ will be picked sufficiently small), \cref{eq:projection-thm} boils down to saying that \[ \inn{h_\cor, f_\cmf} = \inn{\frac{h_\sq}{\norm{h_\sq}}, f_\cmf} = \norm{h_\sq} \leq \sqrt{\epsilon' + 2C\eta}. \]
\end{proof}

We can now put everything together.
\begin{theorem}
	Suppose we have an agnostic learner $L$ for $\cH_{\relu}$ under $\cD$ with a square loss guarantee. Then $L$ can be used to yield a correlation guarantee, i.e.\ to satisfy \cref{as:base-learner}.
\end{theorem}
\begin{proof}
	Run $L$ with $\epsilon' = \Theta(\epsilon^3)$ to get $h$ such that $\norm{h - f_\cmf}^2 \leq \norm{h_\sq - f_\cmf}^2 + \epsilon'$. By \cref{lem:h-small-norm}, if $\norm{h} \leq \eta = \Theta(\epsilon^2)$, then 0 is $\epsilon$-competitive with $h_\cor$. So we may assume that $\norm{h} \geq \Theta(\epsilon^2)$. But then by \cref{eq:h-vs-hcor}, since now $\frac{\epsilon'}{2\norm{h}} \leq \epsilon$, we get that $h/\norm{h}$ is $\epsilon$-competitive with $h_\cor$.
\end{proof}

\end{document}